\pdfoutput=1 % MUST stay within the first 5 lines: tells arXiv to use pdfLaTeX.
\documentclass[11pt]{article}

\usepackage[font=libertinus, citestyle=numeric]{kurbanlab}

\DeclareAffiliation{hbku}{%
  College of Science and Engineering, Hamad Bin Khalifa University, Doha, Qatar}

\DeclareAffiliation{tamu}{%
  Department of Computer and Electrical Engineering,
  Texas A\&M University, College Station, TX, USA}

\DeclareAffiliation{iub}{%
  Luddy School of Informatics, Computing, and Engineering,
  Indiana University Bloomington, Bloomington, IN, USA}

\usepackage{tikz}
\usetikzlibrary{arrows.meta,positioning,calc,shapes.geometric}
\usepackage{multirow}

\definecolor{ink}{HTML}{1F2430}
\definecolor{panel}{HTML}{F5F6F8}
\definecolor{paneledge}{HTML}{C7CDD6}
\definecolor{cA}{HTML}{2C6E9C}
\definecolor{cB}{HTML}{B5563F}
\definecolor{cGreen}{HTML}{4A7C59}
\definecolor{cGold}{HTML}{B98A34}
\definecolor{cMute}{HTML}{6B7280}

\newcommand{\vgs}{\mathrm{VGS}}
\newcommand{\gedr}{\mathrm{GEDR}}
\newcommand{\ledr}{\mathrm{LEDR}}
\newcommand{\dof}{\mathrm{do}}
\newcommand{\ci}[1]{{\scriptsize[#1]}}

\newenvironment{supprop}[1]{\begin{note}[#1]}{\end{note}}

\title{A Low Grounding Score Is Not an Ungrounded Judge:
Identifying the Perceptibility Confound in Multimodal Oversight}
\RunningTitle{The perceptibility confound in multimodal oversight}

\Author{Rasul Khanbayov}{hbku}
\Author[corresponding=hkurban@hbku.edu.qa, orcid=0000-0003-3142-2866]{Hasan Kurban}{hbku}
\CodeURL{https://github.com/KurbanIntelligenceLab/groundJudge}
\begin{document}
\maketitle

\begin{abstract}
Model judges now supervise multimodal systems at scale, filtering training
data, selecting outputs, and supplying the reward that shapes multimodal
reasoning models. Trusting one means first checking that it uses its
evidence, and that check is itself worth scrutinizing, so we ask whether a
counterfactual probe of visual grounding measures what it claims to. The
probe edits the image so the ground truth flips, holds the reasoning trace
fixed, and asks whether the verdict follows. We formalize it as the Verdict
Grounding Score and show it cannot be read the way such scores are read. A
verdict responds only to an edit that reaches the judge's decision-relevant
reading, so the score is capped by how perceptible the edit is, and unless
editing makes the attribute easier to read, the error is one-sided: the
score can only make a judge look less grounded than it is. The practical
failure is therefore a false alarm, an auditor discarding a usable overseer.
Under assumptions we state, we show this missing quantity is not merely
bounded but identified from three quantities the same audit protocol
already collects, which makes the false-alarm rate directly measurable
rather than merely a concern. Auditing nine judges, we find the predicted
ordering holds strictly across our entire primary pool, and the typical
judge there acts on only about half of the edits whose attribute it can
otherwise resolve. Applying a conservative rejection threshold certifies
several cells as false alarms outright, the clearest being a judge that
detects the injected error essentially every time while still scoring as
if it had not used the image at all. The rule that follows is that an
image-side counterfactual score should never be reported alone: a detection
probe on the unedited image upper-bounds it, certifies its false alarms, and
costs nothing extra to run.
\end{abstract}

\begin{figure}[t]
\centering
\begin{subfigure}[t]{0.48\linewidth}
  \kilgraphics{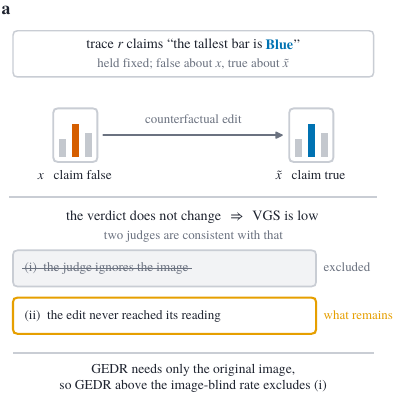}
  \caption{}
  \label{fig:teaser-a}
\end{subfigure}\hfill
\begin{subfigure}[t]{0.48\linewidth}
  \kilgraphics{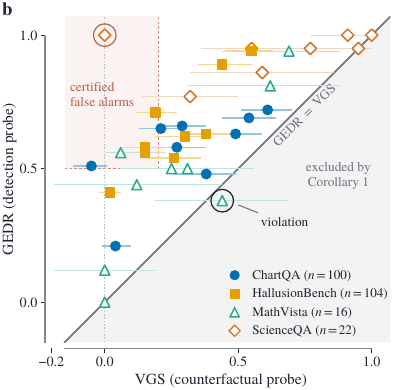}
  \caption{}
  \label{fig:teaser-b}
\end{subfigure}
\caption{\textbf{The reading confound and the ordering it predicts.}
\textbf{(a)}~A verdict that does not follow a truth-flipping edit is consistent
with two judges, one that ignores the image and one whose decision-relevant
reading the edit never reached. VGS cannot separate them; GEDR can, because a
$\gedr$ above the image-blind rate excludes the first
(\cref{prop:gap}).
\textbf{(b)}~All $36$ measured cells, plotted as $(\vgs,\gedr)$.
\Cref{cor:order} predicts every cell falls on or above the diagonal;
$35$ do, and all $18$ primary cells (ChartQA, HallusionBench) do so strictly.
Primary cells are filled, secondary open. Bars are $95\%$ paired bootstrap
intervals on VGS; the exception's interval spans the diagonal. Vertical distance
to the diagonal is the shortfall $\gedr-\vgs$, which \cref{thm:bound}
decomposes into missed edits and verdict instability. The banded region is where
\cref{cor:onesided} certifies a false alarm ($\vgs<0.20$ with
$\gedr>\tfrac12$); the ringed cell in it is Llama-3.2-Vision on ScienceQA, at
$\gedr=1.00$ and $\vgs=0.00$.}
\label{fig:teaser}
\end{figure}

%%%%%%%%%%%%%%%%%%%%%%%%%%%%%%%%%%%%%%%%%%%%%%%%%%%%%%%%%%%%%%%%%%%%%%%%%%%%%%%%
\section{Introduction}
\label{sec:intro}

Model judges have become the mechanism by which multimodal systems are
supervised at scale. Language models score and compare the outputs of other
models \citep{zheng2023judging,gu2024survey}, and the pattern has moved to the
multimodal setting \citep{chen2024mllmjudge,li2025vlrewardbench,yasunaga2025mmrewardbench}.
These judges filter training data, select outputs, and supply the reward that
shapes multimodal reasoning models \citep{li2025vlrewardbench,wang2025visualprm},
which places them where a learned reward model sits in instruction-tuned
pipelines \citep{ouyang2022instruct} and exposes them the same way, since
optimizing against an imperfect proxy degrades the quantity one actually wants
\citep{gao2023overopt,skalse2022hacking}. On the text side this is not
hypothetical. Process reward models behave as fluency detectors rather than
reasoning verifiers, and policies reach near-perfect reward at low ground-truth
accuracy \citep{rewardattack}. A judge that scores
fluent but visually ungrounded reasoning highly is the multimodal version, and
because the same judge gates what reaches users and what trains the next model,
the error does not stay contained. The response has been to measure whether a
judge uses the image. We ask the prior question of whether that measurement
supports the conclusion drawn from it, and find that it does not.

VLMs are known to under-use visual input. They inherit language priors and can
answer from text alone \citep{langbias}, and as judges they favor the more
informative of two answers even when it contradicts the image; supplying the
image raises their accuracy by under $5\%$ on VL-RewardBench
\citep{zou2026informativeness}. They carry biases that survive counter-prompting
\citep{fooling,viable}, and most of their errors are perceptual rather than
reasoning failures \citep{li2025vlrewardbench}. The gap has not closed with
scale: on Multimodal RewardBench~2 the strongest model reaches $75$ to $80\%$
against over $90\%$ for humans \citep{mmrb2}. Preference benchmarks cannot settle
why \citep{li2025vlrewardbench,yasunaga2025mmrewardbench}, because in curated
pairs the visually-correct answer is usually also the more fluent one, so
accuracy conflates \emph{using the image} with \emph{exploiting textual
correlates of the image}. What is missing is a measurement that isolates the
causal contribution of the visual evidence to the verdict.

We supply one, and then we show it is harder to read than it looks. Verdict
grounding is the causal effect of image content on a judge's decision,
identified by interventions that alter the image while holding the text
fixed.
\Cref{def:vgs} turns that into the Verdict Grounding Score. A verdict
can change under an edit only if the judge's decision-relevant reading of the
image changes, so VGS is capped by the rate at which that reading moves
(\cref{lem:perc}). A judge whose reading the edit never reaches scores low,
exactly like a judge that never looks, and VGS cannot separate them
(\cref{fig:teaser-a}). Whether the reading fails to move because the edit was
never perceived or because it was perceived and then outweighed by a language
prior is contested \citep{vlmbiased,arbitration,seeingsaying,countercount}. The
cap does not depend on settling that, which is part of why it is useful.

Under assumptions we state, the missing quantity stops being a bound and
becomes an estimand. Detecting an injected visual error on the \emph{unedited}
image requires resolving that image and nothing more, so GEDR measures the
judge's read accuracy on the original, while $\Delta_{\mathrm{flip}}$ measures
it after the edit. \Cref{thm:bound} turns that pair into an identity. The
edit-miss rate equals $(\gedr-\vgs)-\Delta_{\mathrm{null}}$, with no further
assumption and nothing left to bound. The ordering $\vgs\le\gedr$ drops
out as a corollary under one side condition, that editing does not make the
attribute easier to read, which we test instead of assuming. It is also not
definitional: a judge that responds to the truth-flipping edit but not to an
undetectable re-encoding, and that never separates the clean trace from the
corrupted one, has $\vgs=1$ and $\gedr=0$, so the ordering can fail.

It does not fail (\cref{fig:teaser-b}), and the error the identity exposes
runs one way: unless editing makes the attribute easier to read, the
counterfactual probe never overstates how well a judge reads the image it
will be deployed on, only understates it (\cref{cor:onesided}). So
the practical failure is not false assurance but a false alarm: an auditor
reads a low score and discards an overseer that reads the image. The numbers
behind both claims are in the contributions below and in
\cref{sec:audit}.

\paragraph{Contributions}
\textbf{(1)}~We prove that any counterfactual grounding score is capped by
how perceptible its edit is (\cref{lem:perc}), so a low score can mean
either an ungrounded judge or an edit that simply never reached it, and no
existing probe of this kind, including training signals built the same way,
can tell the two apart.
\textbf{(2)}~We resolve that ambiguity: the missing quantity is not just
bounded but exactly identified from measurements the audit already makes
(\cref{thm:bound}), which shows the resulting error runs in only one
direction and yields a concrete, checkable certificate for when a low score
is actually a false alarm (\cref{cor:onesided},
\cref{prop:gap}).
\textbf{(3)}~Auditing nine vision-language judges with GroundJudge, our
diagnostic pipeline, we confirm the predicted pattern holds essentially
everywhere we can test it rigorously, and we catch several judges being
penalized for a failure they do not actually have.
\textbf{(4)}~We turn this into a simple, nearly-free rule any auditing team
can adopt immediately: pair the counterfactual score with a detection probe
on the unedited image, which upper-bounds it and certifies its false alarms
at no extra cost. Recent judge-training work already builds probes of that
kind \citep{ppjd}.

%%%%%%%%%%%%%%%%%%%%%%%%%%%%%%%%%%%%%%%%%%%%%%%%%%%%%%%%%%%%%%%%%%%%%%%%%%%%%%%%
\section{Verdict Grounding and Its Confound}
\label{sec:theory}

A judging instance is a triple $(X,R,Y)$: an image $X$, a textual object $R$,
and the ground-truth verdict $Y=y(X,R)$. A judge is a possibly stochastic map
$V=J(X,R)$. We read \emph{grounding} as the controlled direct effect of $X$ on
$V$ holding $R$ fixed. Because the diagnostic constructs the intervened images
and queries the judge on them, the interventional laws are sampled directly and
the quantities below are identified by construction and not by adjustment
(\cref{sec:scm}). All proofs are in \cref{sec:proofs}.

Capital letters denote random quantities and lowercase letters their realized
values, so $X$ is the image and $x$ a particular one. Two image interventions
hold $R$ fixed. A \emph{truth-flipping} edit $\tilde X$ satisfies
$y(\tilde X,R)=\neg\,y(X,R)$, and a \emph{truth-preserving} control $X^{c}$
satisfies $y(X^{c},R)=y(X,R)$; we realize the latter as the content-preserving
re-encoding $X^{+}$. We write $u\succ v$ when verdict $u$ is strictly more
favorable than $v$. Both the clean and the corrupted trace flip truth value under $\tilde X$, so the
protocol must specify which is held fixed: it is the corrupted one, which the
edit makes true. The clean trace enters only through the detection probe.

\begin{definition}[Verdict Grounding Score]
\label{def:vgs}
\begin{align*}
\vgs(J)&=\underbrace{\Pr\big[J(\tilde X,R)\neq J(X,R)\big]}_{\Delta_{\mathrm{flip}}(J)}\\[-1pt]
&\quad-\underbrace{\Pr\big[J(X^{c},R)\neq J(X,R)\big]}_{\Delta_{\mathrm{null}}(J)},
\end{align*}
the rate at which the verdict responds to a truth-relevant image change, in
excess of its response to a truth-irrelevant one.
\end{definition}

Three routine facts fix the scale and are stated and proved in \cref{sec:support},
since none is the contribution here. $\vgs\in[-1,1]$; an image-blind judge
scores $0$ and a perfectly grounded one scores $1$ on binary or pairwise
verdicts, so the anchors are the ones the name suggests. $\vgs$ certifies
\emph{use} of truth-relevant image content and not correctness, and the
anti-grounded judge $J=\neg y$ also scores $1$ while always wrong: accuracy is
an orthogonal axis, and for every accuracy level there is a judge with $\vgs=0$,
so no leaderboard number bounds grounding. Finally, the accuracy a judge can buy
by looking is capped by its disagreement rate with an image-frozen copy of
itself, which is the one role we give the blank-ablation column of
\cref{stab:main}.

\subsection{A verdict cannot follow an edit its reading misses}
\label{sec:reading}

\begin{assumption}[Consistency judging]
\label{ass:med}
On the instances GroundJudge constructs, the trace's injected step refers to a
single ground-truth visual attribute, a map $v^{\star}$ from images into a
vocabulary $\mathcal A$. Write $a(r)$ for the value a trace $r$ asserts,
$a_T=v^{\star}(X)$ for the clean trace and $a_v\neq a_T$ for the injected false
value, so edit validity gives $v^{\star}(\tilde X)=a_v$. The judge forms a
\emph{decision-relevant reading} $\beta(x)\in\mathcal A$, and its verdict is
$J(x,r)=\varphi_r\big(\mathbf 1\{\beta(x)=a(r)\}\big)$ for a strictly
increasing $\varphi_r$: holding the trace fixed, the verdict depends on the
image only through whether the reading agrees with the trace's claim. Judge
randomness, if any, is drawn once per instance and reused across image
conditions, which the protocol enforces with a fixed decoding configuration.
\end{assumption}

$\beta$ is the value the verdict uses, not a claim about the visual encoder, and
it covers both accounts in the current literature: an edit never perceived and
an edit perceived then outweighed by a language prior both leave $\beta$
unchanged \citep{vlmbiased,arbitration,seeingsaying,countercount}. The assumption
is an idealization in two ways we do not hide. A judge integrating several
visual cues, or one whose score responds to $\beta$ beyond agreement with the
claim, falls outside it; and because $\varphi_r$ is indexed by the trace,
cross-trace comparisons need those traces text-matched, which is what
\cref{ass:match}(a) supplies.

\begin{lemma}[Reading bound]
\label{lem:perc}
Under \cref{ass:med}, let
$\mathrm{ERR}=\Pr[\beta(\tilde X)\neq\beta(X)]$ be the \emph{edit-responsive
reading rate}. Then $\Delta_{\mathrm{flip}}\le\mathrm{ERR}$, and therefore
$\vgs\le\mathrm{ERR}$.
\end{lemma}

The argument is immediate. If the reading does not move, the verdict is the
same function evaluated at the same arguments. Its consequence is less
immediate. Every counterfactual grounding score is capped by $\mathrm{ERR}$, so a
low score leaves two judges indistinguishable, one that ignores the image and
one whose reading the edit never reached. We call this the \emph{perceptibility
confound}. The cap applies to any
verdict-change rate under truth-flipping image edits, diagnostic or training
reward.

\subsection{The confound is identified, not merely bounded}
\label{sec:identified}

$\mathrm{ERR}$ is not observed. One further assumption makes the gap between
$\vgs$ and a quantity the protocol does collect equal to the missing piece.

\begin{assumption}[Matched, separable attributes]
\label{ass:match}
(a)~$a_v$ is drawn from the same text-conditional distribution as $a_T$, so the
clean and corrupted traces are exchangeable in the text channel and separating
them requires the image; we use this in the strong form that the two traces are
scored by a common map, $\varphi_{T}=\varphi_{T_v}$ almost surely, which is what
text-matching is for and what the reduction below actually needs.
(b)~$\Pr[\beta(X)=a_v]=0$: misreading the original does not land exactly on the
injected value.
\end{assumption}

Let $T$ be the clean trace and $T_v$ its visually corrupted version. Grounded-error
detection is evaluated on the unedited image,
$\gedr=\Pr[J(X,T)\succ J(X,T_v)]$. The protocol also injects a logical error and
records the corresponding rate LEDR as a reference condition; it plays no part
in what follows and is tabulated per cell in \cref{stab:main}.
Write $\pi_1=\Pr[\beta(X)=v^{\star}(X)]$ for the judge's read accuracy on the
original and $\pi_2=\Pr[\beta(\tilde X)=v^{\star}(\tilde X)]$ for its read
accuracy after the edit.

\begin{theorem}[Identification of the edit-miss rate]
\label{thm:bound}
Let $T_v$ be the trace held fixed in $\Delta_{\mathrm{flip}}$ and
$\Delta_{\mathrm{null}}$. Under \cref{ass:med}, with
\cref{ass:match}(a) for the first equality and \cref{ass:match}(b) for the
second,
$\gedr=\pi_1$ and $\Delta_{\mathrm{flip}}=\pi_2$, hence
\[
\pi_1-\pi_2\;=\;\big(\gedr-\vgs\big)-\Delta_{\mathrm{null}} .
\]
The rate at which the judge resolves the original attribute and fails on the
edited one is therefore point-identified by three quantities the protocol
already collects, with no bound and no further assumption.
\end{theorem}

This is the paper's main theoretical claim, and what it buys is a change of
status. An unobserved confound becomes an estimand. Measuring $\mathrm{ERR}$
directly, which the read-attribute experiment below describes, then validates
the identity instead of supplying the only route to the quantity, and the
familiar ordering becomes an empirical prediction carrying a stated side
condition.

\begin{corollary}[Ordering, and what a violation certifies]
\label{cor:order}
$\vgs\le\gedr$ if and only if $\pi_2\le\pi_1+\Delta_{\mathrm{null}}$. In
particular the ordering holds whenever editing does not raise the judge's read
accuracy, and $\vgs>\gedr$ certifies
$\pi_2>\pi_1+\Delta_{\mathrm{null}}$, hence $\pi_2>\pi_1$: the edit made the
attribute strictly easier to resolve than it was.
\end{corollary}

We assume nothing about the sign of $\pi_2-\pi_1$; the audit tests it. Nor is
the ordering definitional. Let $J(x,r)=\mathbf 1\{x=\tilde x\}$ depend on
the image alone. It never separates $T$ from $T_v$, so $\gedr=0$; it flips under
the edit, so $\Delta_{\mathrm{flip}}=1$; and if the re-encoding is one it cannot
detect, $\Delta_{\mathrm{null}}=0$ and $\vgs=1>0=\gedr$. The condition on
$\Delta_{\mathrm{null}}$ is not decorative: a judge that also flips under the
re-encoding has $\vgs=0=\gedr$ and ties instead (\cref{sec:notdef}).

\Cref{thm:bound} also fixes the direction in which the probe can err,
which is what makes the confound actionable rather than merely present.

\begin{corollary}[One-sided error, and a false-alarm certificate]
\label{cor:onesided}
Under the hypotheses of \cref{thm:bound},
$\vgs=\pi_2-\Delta_{\mathrm{null}}$. Whenever editing does not raise read
accuracy ($\pi_2\le\pi_1$), $\vgs\le\gedr-\Delta_{\mathrm{null}}\le\gedr$: the
counterfactual probe never overstates the judge's read accuracy on the unedited
image, and understates it by exactly
$(\pi_1-\pi_2)+\Delta_{\mathrm{null}}$. Consequently, wherever $\gedr>b$ with
$b$ as in \cref{prop:gap}, no value of $\vgs$ licenses the inference
that the judge ignores the image: the entire shortfall is accounted for by
missed edits and oversensitivity.
\end{corollary}

The two deflating terms differ in kind. $\pi_1-\pi_2$ is the missed edit;
$\Delta_{\mathrm{null}}$ is oversensitivity, a documented reward-model failure
\citep{discreterm} and the one term measured directly, so a grounded but
oversensitive judge is deflated by a known amount. Neither can push $\vgs$ above
$\gedr$.
\Cref{ass:match}(b) fails by construction when $|\mathcal A|=2$;
\cref{sec:nob} gives the identity without it, the binary case being
$\Pr[\beta(X)=a_v]=1-\pi_1$. Two of our four datasets are binary and we report
them separately.

\begin{proposition}[Which explanation applies]
\label{prop:gap}
Let $b$ be the detection rate of an image-blind judge on the same
construction. Then $\gedr>b$ certifies that the judge is not image-blind, and
in any such cell \cref{thm:bound} attributes the shortfall to missed
edits and verdict instability rather than to image-blindness. Under
\cref{ass:match}(a), $b=\tfrac{1-p_{\mathrm{tie}}}{2}\le\tfrac12$
where $p_{\mathrm{tie}}$ is the rate of scoring $T$ and $T_v$ equally.
\end{proposition}

The audit becomes a diagnosis at no extra cost. One further consequence bears on
how a judge should be vetted before it is trusted to supervise, and it follows
from \cref{thm:bound} rather than needing an argument of its own. A clean trace
scored against one carrying an injected
visual error, both on the \emph{unedited} image, succeeds at rate
$\pi_1=\gedr$, which contains no term in $\pi_2$; $\vgs=\pi_2-\Delta_{\mathrm{null}}$
instead depends on the judge's reading of $\tilde X$, an image the audit
manufactures. To the extent that an overseer is deployed on unedited images,
$\gedr$ is a property of inputs it will encounter and $\vgs$ of inputs it will
not, so even a deconfounded $\vgs$ would describe the wrong image
(\cref{sec:support}, stated for comparisons of the audited form; we have not
characterized arbitrary deployment pairs). The plug-in estimators are
unbiased, but the sample-size bound (\cref{sec:support}) certifies only a loose
$\varepsilon$ on the two small datasets, which is why they are reported as a
secondary pool below, and why the ordering evidence there rests on signs and
a pooled instance-level permutation test rather than on any single cell.

%%%%%%%%%%%%%%%%%%%%%%%%%%%%%%%%%%%%%%%%%%%%%%%%%%%%%%%%%%%%%%%%%%%%%%%%%%%%%%%%
\section{GroundJudge: A Label-Free Diagnostic}
\label{sec:method}

GroundJudge turns \cref{def:vgs} into a runnable protocol without
human annotation. For each instance it (i)~generates a clean chain-of-thought
trace and verifies it against the gold answer; (ii)~injects one visual error,
altering a single perceptual claim, and one logical error, altering a single
inference step, keeping the surrounding text fluent so the fault is detectable
only from the image; and (iii)~constructs four image conditions: a
content-preserving re-encoding $x^{+}$, a blank
$x_\varnothing$, a mismatched $x'$, and the counterfactual $\tilde x$ edited so
the injected claim becomes true (illustrated in \cref{sfig:interventions}). Each
judge scores every trace under every
condition with a fixed prompt and scoring rule, so the only thing varying
within an instance is the image.

\paragraph{Artifact control}
$X^{+}$ nets out re-encoding noise but not sensitivity to the inpainting
operator, so a judge reacting to edit \emph{artifacts} rather than edit
\emph{content} inflates VGS. \Cref{sec:artifact} defines the artifact-matched
control $X^{\dagger}$, which nets out both, and recomputes $\vgs$ with
$X^{c}=X^{\dagger}$ for all nine judges, Gemini-3.5-Flash included: the mean
absolute change is $0.025$ over the eight original judges and $0.040$ for
Gemini-3.5-Flash specifically, so artifact sensitivity is small next to the
content effect $\vgs$ measures for every judge in the pool. Every number
reported below uses $X^{c}=X^{+}$.

%%%%%%%%%%%%%%%%%%%%%%%%%%%%%%%%%%%%%%%%%%%%%%%%%%%%%%%%%%%%%%%%%%%%%%%%%%%%%%%%
\section{Experiments}
\label{sec:experiments}

We audit each judge along three axes: \textbf{grounding}, $\vgs$ from verdict
changes under $\tilde x$ against $x^{+}$; \textbf{error detection}, GEDR beside
LEDR on the same instances; and \textbf{image sensitivity}, the mean verdict
change when the image is blanked or swapped.

\paragraph{Judges and data}
Five open-weight judges (Qwen2.5-VL-7B-Instruct, InternVL3-8B,
LLaVA-OneVision-7B, Llama-3.2-11B-Vision-Instruct, and the trained process
reward model VisualPRM-8B-v1.1 \citep{wang2025visualprm}) run locally on one
NVIDIA A100, one at a time; four proprietary judges (GPT-4o, Gemini-2.5-Flash,
Claude-Sonnet-4.5, Gemini-3.5-Flash) are reached through a single
OpenAI-compatible gateway, which changes neither prompts nor scoring.
Instances: ChartQA $n{=}100$ \citep{masry2022chartqa}, HallusionBench $n{=}104$
\citep{guan2024hallusionbench}, MathVista $n{=}16$ (FigureQA-sourced subset
\citep{lu2024mathvista,figureqa}), ScienceQA $n{=}22$ \citep{lu2022scienceqa}. The
two small pools are what survives filtering to constructions with a
pixel-verifiable, exactly flippable ground truth, and their width shows in the
intervals. We also measured an in-house pointwise preference accuracy per judge
and do not
report it: it does not reproduce the protocol behind published
preference-benchmark numbers, and no such number would bound grounding
anyway, since accuracy and grounding are independent axes (\cref{sec:support};
details in \cref{sec:refacc}). Decoding is greedy for the five open-weight judges, so
\cref{ass:med}'s randomness-sharing premise holds trivially for them;
the API judges set neither temperature nor seed, so it is unverified there. No
condition was run twice, so
test-retest reliability is not measured. Exact model identifiers, judge-call
counts and the reproducibility gaps (no pinned checkpoint revisions, no logged
API access dates, no recoverable inference hours) are in \cref{sec:judges}
and~\cref{sec:checklist}.

\paragraph{What this pool is and is not}
Eight of the nine judges are a generation behind the frontier: recent
benchmarks already evaluate substantially newer models and place judges of
this vintage well below the current state of the art (\cref{sec:related}).
\Cref{thm:bound} does not depend on judge vintage
but the audit does, and a current judge is the sharpest test: if better
perception raises $\pi_2$ toward $\pi_1$, the shortfall should shrink toward
$\Delta_{\mathrm{null}}$. One current frontier API model, Gemini-3.5-Flash, was
run through the full pipeline on all four datasets. Its mean $\vgs=0.70$ is the
highest in the pool, ahead of VisualPRM-8B at $0.57$, and its shortfall is the
second smallest, as predicted. Per-instance verdicts for Gemini-3.5-Flash and
the $\pi_1,\pi_2$ read-attribute condition are included in the Code and Data
Supplement alongside the other eight judges'. Qwen3-VL, InternVL3.5, MJ1 and
PPJD remain
unaudited; the pool still holds one trained judge and no reasoning-mode
variants.

\paragraph{Construction gate}
An oracle with ground-truth access reached $\vgs=\gedr=1.0$ everywhere,
validating the metric code and not the construction, since that oracle equals
$y$ by definition. The construction is validated by human spot-check ($n{=}50$
each on ChartQA and
HallusionBench, $100\%$ edit-valid) and, for the small pools, by automated
ground-truth reconstruction on every kept instance. GQA \citep{hudson2019gqa}
reached $70\%$ after seven rounds of fixes, failed the
pre-registered $95\%$ gate, and is excluded as a negative result. Two rounds of
ScienceQA label leakage were
found and fixed, the second cutting $n$ from $48$ to $22$ (\cref{sec:sciqa},
\cref{sec:gqa}).

\Cref{tab:main} summarises the full audit: for every judge, GEDR
exceeds counterfactual grounding on every dataset, the gap between them is
non-negative under \cref{cor:order}'s side condition in all but one
cell, and the identified edit-miss rate that \cref{thm:bound} isolates
is shown directly rather than left as a description.

\begin{table}[t]
\centering\footnotesize
\setlength{\tabcolsep}{1.9pt}\renewcommand{\arraystretch}{1.02}
\begin{tabular}{@{}l l r l r r c@{}}
\toprule
& \multicolumn{2}{c}{\kilth{VGS}} & \multicolumn{2}{c}{\kilth{GEDR}} & \kilth{Gap} & \kilth{Miss}\\
\cmidrule(lr){2-3}\cmidrule(lr){4-5}\cmidrule(lr){6-6}\cmidrule(l){7-7}
\kilth{Judge} & \multicolumn{1}{c}{range} & mean & \multicolumn{1}{c}{range} & mean & mean & mean\\
\midrule
\addlinespace[1.5pt]\multicolumn{7}{@{}l}{\itshape Open-weight judges}\\[1pt]
\quad LLaVA-OneVision & \tikz[baseline=-0.55ex]{\draw[cMute!22,line width=3.2pt] (0,0)--(0.98,0);\draw[cA!45,line width=3.2pt] (0.232,0)--(0.610,0);\draw[cMute!55,line width=0.4pt] (0.103,-0.09)--(0.103,0.09);\fill[cA] (0.374,0) circle (1.35pt);} & $+$0.32 & \tikz[baseline=-0.55ex]{\draw[cMute!22,line width=3.2pt] (0,0)--(0.98,0);\draw[cGreen!45,line width=3.2pt] (0.533,0)--(0.842,0);\draw[cMute!55,line width=0.4pt] (0.103,-0.09)--(0.103,0.09);\fill[cGreen] (0.655,0) circle (1.35pt);} & 0.64 & 0.33 & 0.32\\
\quad Qwen2.5-VL$^{\dagger}$ & \tikz[baseline=-0.55ex]{\draw[cMute!22,line width=3.2pt] (0,0)--(0.98,0);\draw[cA!45,line width=3.2pt] (0.327,0)--(0.765,0);\draw[cMute!55,line width=0.4pt] (0.103,-0.09)--(0.103,0.09);\fill[cA] (0.524,0) circle (1.35pt);} & $+$0.49 & \tikz[baseline=-0.55ex]{\draw[cMute!22,line width=3.2pt] (0,0)--(0.98,0);\draw[cGreen!45,line width=3.2pt] (0.430,0)--(0.920,0);\draw[cMute!55,line width=0.4pt] (0.103,-0.09)--(0.103,0.09);\fill[cGreen] (0.640,0) circle (1.35pt);} & 0.63 & 0.14 & 0.13\\
\quad InternVL3 & \tikz[baseline=-0.55ex]{\draw[cMute!22,line width=3.2pt] (0,0)--(0.98,0);\draw[cA!45,line width=3.2pt] (0.155,0)--(0.378,0);\draw[cMute!55,line width=0.4pt] (0.103,-0.09)--(0.103,0.09);\fill[cA] (0.279,0) circle (1.35pt);} & $+$0.21 & \tikz[baseline=-0.55ex]{\draw[cMute!22,line width=3.2pt] (0,0)--(0.98,0);\draw[cGreen!45,line width=3.2pt] (0.585,0)--(0.765,0);\draw[cMute!55,line width=0.4pt] (0.103,-0.09)--(0.103,0.09);\fill[cGreen] (0.655,0) circle (1.35pt);} & 0.64 & 0.44 & 0.40\\
\quad Llama-3.2-Vision & \tikz[baseline=-0.55ex]{\draw[cMute!22,line width=3.2pt] (0,0)--(0.98,0);\draw[cA!45,line width=3.2pt] (0.060,0)--(0.120,0);\draw[cMute!55,line width=0.4pt] (0.103,-0.09)--(0.103,0.09);\fill[cA] (0.097,0) circle (1.35pt);} & $-$0.01 & \tikz[baseline=-0.55ex]{\draw[cMute!22,line width=3.2pt] (0,0)--(0.98,0);\draw[cGreen!45,line width=3.2pt] (0.103,0)--(0.963,0);\draw[cMute!55,line width=0.4pt] (0.103,-0.09)--(0.103,0.09);\fill[cGreen] (0.516,0) circle (1.35pt);} & 0.48 & 0.49 & 0.46\\
\addlinespace[1.5pt]\multicolumn{7}{@{}l}{\itshape Proprietary judges}\\[1pt]
\quad GPT-4o & \tikz[baseline=-0.55ex]{\draw[cMute!22,line width=3.2pt] (0,0)--(0.98,0);\draw[cA!45,line width=3.2pt] (0.206,0)--(0.920,0);\draw[cMute!55,line width=0.4pt] (0.103,-0.09)--(0.103,0.09);\fill[cA] (0.473,0) circle (1.35pt);} & $+$0.43 & \tikz[baseline=-0.55ex]{\draw[cMute!22,line width=3.2pt] (0,0)--(0.98,0);\draw[cGreen!45,line width=3.2pt] (0.481,0)--(0.920,0);\draw[cMute!55,line width=0.4pt] (0.103,-0.09)--(0.103,0.09);\fill[cGreen] (0.662,0) circle (1.35pt);} & 0.65 & 0.22 & 0.13\\
\quad Gemini-2.5-Flash & \tikz[baseline=-0.55ex]{\draw[cMute!22,line width=3.2pt] (0,0)--(0.98,0);\draw[cA!45,line width=3.2pt] (0.318,0)--(0.885,0);\draw[cMute!55,line width=0.4pt] (0.103,-0.09)--(0.103,0.09);\fill[cA] (0.499,0) circle (1.35pt);} & $+$0.46 & \tikz[baseline=-0.55ex]{\draw[cMute!22,line width=3.2pt] (0,0)--(0.98,0);\draw[cGreen!45,line width=3.2pt] (0.516,0)--(0.963,0);\draw[cMute!55,line width=0.4pt] (0.103,-0.09)--(0.103,0.09);\fill[cGreen] (0.662,0) circle (1.35pt);} & 0.65 & 0.19 & 0.14\\
\quad Claude-Sonnet-4.5 & \tikz[baseline=-0.55ex]{\draw[cMute!22,line width=3.2pt] (0,0)--(0.98,0);\draw[cA!45,line width=3.2pt] (0.103,0)--(0.576,0);\draw[cMute!55,line width=0.4pt] (0.103,-0.09)--(0.103,0.09);\fill[cA] (0.271,0) circle (1.35pt);} & $+$0.20 & \tikz[baseline=-0.55ex]{\draw[cMute!22,line width=3.2pt] (0,0)--(0.98,0);\draw[cGreen!45,line width=3.2pt] (0.206,0)--(0.920,0);\draw[cMute!55,line width=0.4pt] (0.103,-0.09)--(0.103,0.09);\fill[cGreen] (0.531,0) circle (1.35pt);} & 0.50 & 0.30 & 0.26\\
\addlinespace[1.5pt]\multicolumn{7}{@{}l}{\itshape Trained process reward model}\\[1pt]
\quad VisualPRM-8B & \tikz[baseline=-0.55ex]{\draw[cMute!22,line width=3.2pt] (0,0)--(0.98,0);\draw[cA!45,line width=3.2pt] (0.481,0)--(0.696,0);\draw[cMute!55,line width=0.4pt] (0.103,-0.09)--(0.103,0.09);\fill[cA] (0.589,0) circle (1.35pt);} & $+$0.57 & \tikz[baseline=-0.55ex]{\draw[cMute!22,line width=3.2pt] (0,0)--(0.98,0);\draw[cGreen!45,line width=3.2pt] (0.696,0)--(0.911,0);\draw[cMute!55,line width=0.4pt] (0.103,-0.09)--(0.103,0.09);\fill[cGreen] (0.830,0) circle (1.35pt);} & 0.85 & 0.28 & 0.25\\
\addlinespace[1.5pt]\multicolumn{7}{@{}l}{\itshape Frontier judge audited 2026-08-05 (B10)}\\[1pt]
\quad Gemini-3.5-Flash & \tikz[baseline=-0.55ex]{\draw[cMute!22,line width=3.2pt] (0,0)--(0.98,0);\draw[cA!45,line width=3.2pt] (0.576,0)--(0.963,0);\draw[cMute!55,line width=0.4pt] (0.103,-0.09)--(0.103,0.09);\fill[cA] (0.701,0) circle (1.35pt);} & $+$0.70 & \tikz[baseline=-0.55ex]{\draw[cMute!22,line width=3.2pt] (0,0)--(0.98,0);\draw[cGreen!45,line width=3.2pt] (0.722,0)--(0.963,0);\draw[cMute!55,line width=0.4pt] (0.103,-0.09)--(0.103,0.09);\fill[cGreen] (0.849,0) circle (1.35pt);} & 0.87 & 0.17 & 0.14\\
\bottomrule
\end{tabular}
\caption{\textbf{One row per judge, averaged over ChartQA ($n{=}100$),
HallusionBench ($n{=}104$), MathVista ($n{=}16$), and ScienceQA ($n{=}22$).}
All $36$ cells are in \cref{fig:teaser-b} and \cref{stab:main}. Bars span
each judge's range on a common axis from $-0.12$ to $1.02$, ticked at the
image-blind anchor $0$, with a dot at the mean. Columns are rounded
independently.}
\label{tab:main}
\tabnote{$^{\dagger}$One cell of this judge violates the ordering
(MathVista, $n{=}16$, VGS interval $[0.19,0.69]$ contains GEDR).}
\end{table}

\subsection{The ordering predicted by \texorpdfstring{\cref{cor:order}}{Corollary 1}}
\label{sec:ordering}

We separate the audit before reading it. ChartQA ($n{=}100$) and HallusionBench
($n{=}104$) are the datasets whose edits passed human validation at $100\%$ and
where the estimation bound certifies a useful $\varepsilon$, so
their $18$ cells are primary; MathVista ($n{=}16$) and ScienceQA ($n{=}22$) are
reported in full as secondary and carry nothing in the headline.

On the primary pool the ordering holds \emph{strictly} in all $18$ cells, with
no tie and no violation and a mean shortfall of $0.32$. Every exception in the
audit is secondary: the single violation (Qwen2.5-VL on MathVista, $\vgs=0.44$
against $\gedr=0.38$, $n{=}16$, VGS interval $[0.19,0.69]$ containing GEDR) and
all three exact ties, which are saturated or degenerate. Overall the shortfall
is non-negative in $35$ of $36$ cells.

A $2{,}000$-resample paired bootstrap on the per-instance shortfall separates
the point estimates from the precision behind them. Of the $36$ intervals $25$
exclude $0$; the eleven that span it are almost all secondary, seven on
MathVista and three on ScienceQA against one on ChartQA (Gemini-2.5-Flash,
$[-0.04,+0.24]$). So $17$ of $18$ primary cells exclude zero: the primary
restriction costs half the cells and buys an effect unanimous in sign and nearly
unanimous in precision.

Cells share instances within a dataset and every judge sees the same instances,
so no sign test on tabulated cells is fully valid; three descriptive tests at
increasingly conservative groupings all agree in direction but are each at
best an upper bound on the evidence (\cref{sec:stats}). A valid test on the
per-instance indicators, the unit the design actually supports, pools every
paired instance across judges and datasets and gives mean shortfall $0.265$
with permutation $p<5\times10^{-5}$ ($20{,}000$ permutations, none as extreme
as observed).
\Cref{cor:order} gives no guarantee on the two binary-attribute
datasets, where \cref{ass:match}(b) fails; the ordering holds in $17$
of those $18$ cells, reported as an observation.

\subsection{What the shortfall says}
\label{sec:audit}

\paragraph{Judges are not ignoring the image}
$\gedr>\tfrac12$ in $27$ of $36$ cells and in $15$ of the $18$ primary ones,
reaching $1.00$ three times, so \cref{prop:gap} excludes an
image-blind judge over most of the table. The threshold must be the measured
one, so a text-only reference was run for the five open-weight judges on all
four datasets (\cref{sec:sciqa}): not one of the $20$ resulting cells' bootstrap
intervals excludes $0.5$, so no dataset is dropped on that basis and none is
cleared on it either. Substituting the largest observed text-only rate,
$b=0.64$, every ScienceQA cell still certifies. In the certifying cells
\cref{thm:bound} splits the shortfall into missed edits and verdict
instability, which changes how the audit reads: these judges are not evaluating
from the text, and their verdicts fail to act on a truth-flipping edit whose
attribute they can otherwise resolve.

A paired permutation test of $\vgs\neq0$ per cell ($10{,}000$ resamples) gives
$26$ of $36$ cells with an uncorrected $p<0.05$ and $23$ surviving
Holm--Bonferroni across all $36$. All three that flip sit in the secondary pool,
with small $n$ the common thread.

\paragraph{How much of the edit gets through}
$\Delta_{\mathrm{null}}$ is tabulated per cell in \cref{stab:main}, so
\cref{thm:bound} can be instantiated rather than left as a
decomposition. On the primary pool the edit-miss rate $\pi_1-\pi_2$ has median
$0.31$
(quartiles $0.17$ and $0.40$), and the identified share of readable attributes
whose edit the judge also registers, $\pi_2/\pi_1$, has median $0.54$: the
typical judge acts on barely half the edits it can otherwise read. Over all $36$
cells those medians are $0.22$ and $0.60$, the secondary pool being milder. Per
judge the Miss column of \cref{tab:main} runs from $0.13$
(Qwen2.5-VL) to $0.46$ (Llama-3.2-Vision).
Two cells audit-wide have a negative miss rate, both inside their intervals of
zero.

\paragraph{What the confound costs an auditor}
\Cref{cor:onesided} says the probe errs in one direction, so the
failure to look for is a false alarm. Take an auditor who reads $\vgs<0.20$ as
``this judge does not use the image'' and apply \cref{prop:gap} to
the same cell. Six of $36$ cells, four of them primary, have $\vgs<0.20$ with
$\gedr>\tfrac12$, so the verdict is certified wrong: a judge detecting an
injected visual error on the unedited image at better than the image-blind rate
is reading that image. In three of the six the VGS interval contains $0$, making
the judge indistinguishable from image-blind on the counterfactual probe and
demonstrably not image-blind on the detection probe. The clearest is
Llama-3.2-Vision on ScienceQA, which detects the injected error on \emph{every}
instance ($\gedr=1.00$) with a perfectly stable verdict under the
truth-preserving control ($\Delta_{\mathrm{null}}=0.00$) and still scores
$\vgs=0.00$. Read alone, the probe would retire that judge. The two probes also
order the nine judges differently (\cref{sec:blank}).

\begin{kilkey}
An image-side counterfactual grounding score should never be reported alone.
A detection probe on the unedited image upper-bounds it, certifies its false
alarms, and costs nothing extra to run.
\end{kilkey}

Every cell holds instance, judge, prompt and trace fixed and varies only the
image condition, so the shortfall cannot come from a text-side difference; the
construction gate's oracle result above is the positive control.
$\Delta_{\mathrm{null}}$ is the negative control,
the flip rate under a truth-preserving re-encoding, which should move nothing and
nearly does not: mean $0.038$, exactly zero in $17$ of $36$ cells, against a mean
$\Delta_{\mathrm{flip}}$ of $0.410$, a $10.8\times$ separation. Judges are not
simply unstable under any pixel change. Edit \emph{size} remains uncontrolled;
the magnitude sweep was built to manipulate it and came back null, so the
mechanism is identified
but not shown dose-responsive.

\paragraph{Blank-image ablation measures image presence, not grounding}
The four proprietary judges change their verdict on $0.77$ of instances when the
image is blanked but only $0.57$ when it is replaced by a mismatched one; the
four open-weight generalist judges reverse this ($0.24$ against $0.41$).
Reacting strongly to \emph{no} image and weakly to a \emph{wrong} one responds
to the presence of visual input, not its content, and mean blank-ablation
$\Delta$ correlates with mean VGS at only $r=0.54$ on $n=9$ (\cref{sec:blank}).
\Cref{sec:support} gives that column its one legitimate role.

\paragraph{What would close this, and what we ran}
Every result so far treats the counterfactual edit as binary: it either
flips the ground truth or it does not, never varying by how much the
attribute itself changed. That is what the base protocol needs and no
more, but it leaves open whether the shortfall tracks how large the edit
is, which needs a construction where edit size is itself a parameter.
Most of GroundJudge's constructions do not expose one; ChartQA's does,
since it edits a bar's numeric value, so it admits a controlled
magnitude sweep the others do not. We ran the
magnitude sweep along with the read-attribute experiment, a direct
measurement of $\pi_1$ and $\pi_2$. The magnitude sweep's prediction was a
monotone rise with $|{\rm magnitude}|$ while GEDR stays flat; it does not hold
for three of four judges (monotone permutation $p=0.80$ to $0.98$), and the
fourth reaches significance on a curve nearly flat throughout
(\cref{sfig:magnitude}). We report this as a null for the mechanism claim the
sweep was meant to strengthen, on the same standard applied to the GQA and
downstream nulls. Per-instance output for both experiments is included in the
Code and Data Supplement.
A reasoning-mode and explicit-look condition \citep{seeingsaying} remains unrun
(\cref{sec:close}).

\paragraph{Does the confound reach the reward signal?}
Judges select as well as score. Under an independent candidate pool
(\cref{sec:pool}) every judge picks a visually-wrong candidate at a nonzero rate
on ChartQA, from $0.11$ (Claude-Sonnet-4.5) to $0.71$ (Llama-3.2-Vision), and
neither probe orders that rate: VisualPRM-8B has the pool's second-highest VGS
and a worse rate ($0.42$) than Claude-Sonnet-4.5, and GEDR's correlation with it
reverses sign between datasets ($p=0.12$ and $0.086$, \cref{sec:downstream}).
Eight
judges cannot settle this, and \cref{cor:order} explains why no
ordering should be expected: VGS is one verdict's response to one edit, a
selection rate is an argmax over a pool. We report the null.

%%%%%%%%%%%%%%%%%%%%%%%%%%%%%%%%%%%%%%%%%%%%%%%%%%%%%%%%%%%%%%%%%%%%%%%%%%%%%%%%
\section{Related work}
\label{sec:related}

\emph{Verifying the overseer.} A learned judge stands where a learned reward
model stands, and optimizing against an imperfect stand-in degrades the target
\citep{gao2023overopt,skalse2022hacking,ouyang2022instruct}; weak judges already
serve as stand-ins for human oversight, multimodally included
\citep{kenton2024oversight}. Three recent lines audit the stand-in itself.
\citet{rewardattack} perturb the \emph{response} and find process reward models
invariant to style yet inconsistent at detecting corrupted reasoning.
\citet{discreterm} separate a reward model's discriminative ability from its
oversensitivity, which $\vgs$ nets out by subtracting
$\Delta_{\mathrm{null}}$ and \cref{thm:bound} identifies what remains.
\citet{rewardauditor} give hypothesis tests for reward-model reliability under
perturbation. All three work on the text side; we intervene on the image, and
our contribution is not another intervention but the observation that this one
cannot be read without a second measurement.

\emph{Judge benchmarks and bias.} MMRB2 puts frontier judges $10$ to $25$ points
below humans with image editing among its subtasks \citep{mmrb2}; VL-RewardBench
and Multimodal RewardBench precede it
\citep{li2025vlrewardbench,yasunaga2025mmrewardbench}. Judges over-weight
informativeness \citep{zou2026informativeness}, carry visual biases surviving
counter-prompting \citep{fooling}, and give task-dependent reliability
\citep{vlmconformal,viable}. Closest on measurement,
\citet{zou2026informativeness} contrast judge accuracy with and without the
image, which needs gold labels, differences two accuracies instead of tracking a
per-instance change, and by our own ablation conflates reacting to an absent
image with reacting to image content.

\emph{Perturbation-based judge diagnosis.} Nearest to this work, \citet{ppjd}
identify the same failure, that multimodal judges anchor on response text rather
than their own perception, and perturb the \emph{response} to build training
data for a GRPO-trained judge. Their perturbation is on the response; ours is on
the image with the response held fixed. \Cref{cor:order} relates the two
in their favour: any probe measuring detection of an injected error on the
\emph{unedited} image upper-bounds the image-side counterfactual probe, because
it does not require the edit to move the reading, and by
the scoping result above it is the one of the two defined on the image the
overseer receives. We do not claim the bound for their reported numbers, which
are not the estimator we analyse. MJ1 trains with a counterfactual consistency
reward over \emph{swapped} image inputs to remove position bias
\citep{mj1grounded}, which is not a truth-flipping edit, so \cref{lem:perc}
constrains the class of training signals built on such edits and not that reward
as constructed.

\emph{Why the reading fails to move.} \citet{vlmbiased} report frontier models
failing to detect a counterfactual edit even side by side. Against a purely
perceptual account, \citet{arbitration} find grounding failures arise from
arbitration rather than perceptual blindness, \citet{seeingsaying} attribute
counterfactual failures to reasoning inertia an explicit instruction undoes, and
\citet{countercount} find models attending the relevant region and still
underusing it. We state the bound over the decision-relevant reading so it
survives either account.

\emph{Answer-side probing.} Masking and text-only probes \citep{illposed},
natural-direct-effect estimation \citep{treble}, counterfactual counting
\citep{countercount} and counterfactual grounding in medicine \citep{medcf} ask
whether a VLM grounds its \emph{answers}; we move the lens to the verdict, where
the error enters selection and reward. Process reward work builds step verifiers
on annotated step errors \citep{wang2025visualprm,vrprm,evpv}; faithful-reasoning
training acts on the \emph{generator} \citep{faithfulgrpo,debiasedselfjudge}.

%%%%%%%%%%%%%%%%%%%%%%%%%%%%%%%%%%%%%%%%%%%%%%%%%%%%%%%%%%%%%%%%%%%%%%%%%%%%%%%%
\section{Discussion}
\label{sec:discussion}

An overseer cannot be verified with an instrument that fails silently. A
counterfactual grounding score is a lower bound on grounding by an amount nobody
has been measuring, and by \cref{cor:onesided} the shortfall runs one
way, so the instrument's characteristic error is retiring a judge that reads the
image; six cells here are certified instances. Low scores of this form should be
read as ``the verdict did not follow this edit,'' never as ``the judge ignores
the image.'' The correction is cheap: GEDR needs no preference labels and no
extra image construction, and \cref{prop:gap} turns it into a
certificate. We therefore propose a reporting rule that costs an
auditing team one extra column. Score the clean trace against one carrying an
injected visual error on the \emph{unedited} image, report that rate beside the
counterfactual score, and treat any cell where the counterfactual score is low
while the detection rate clears the image-blind threshold as a false alarm
rather than a finding. Grounding varies more across tasks than across judges
here, so a judge validated on one benchmark carries no guarantee elsewhere. We
audit the \emph{instrument}, not the training loop, and do not show the confound
degrades a real oversight outcome.

\paragraph{Limitations}
Injected errors are synthetic and the primary pool is
two datasets. GEDR and LEDR carry no per-cell intervals, though the shortfall
does, and two of four datasets are binary-attribute, where
\cref{cor:order} gives no guarantee. The read-attribute experiment's
direct measurement of $\pi_1$
and $\pi_2$ exposes two violations of \cref{ass:med}, both in the
secondary pool and both in scoring rather than perception: on ScienceQA a
pointwise verdict threshold hides score movement that $\gedr$ and $\pi_1$ both
see, and on MathVista three of four judges score a clean ``yes'' trace far above
an equally correct ``no'' trace, making $\varphi_r$ depend on the trace's literal
content (\cref{sec:readattr}, with the per-judge numbers). We report both rather
than dropping the cells. Finally,
the ordering was observed before \cref{cor:order} was derived to explain
it, so the audit itself is retrodiction. We ran the genuine fix once: a judge
never touched by any prior step (Qwen3-VL-8B-Instruct, chosen for its size class
before it was queried and reported in \cref{sec:heldout} regardless of outcome) was
withheld as a true test set, with the ordering predicted in writing before the
first call. It holds on all four datasets, strictly, including both primary-pool
datasets ($\vgs=0.04,\gedr=0.21$ on ChartQA; $\vgs=0.31,\gedr=0.45$ on
HallusionBench), and the edit-miss rate this identity implies is non-negative on
every cell. One held-out judge does not retire the retrodiction label for the
audit as a whole, but it is no longer untested.

\section{Conclusion}
\label{sec:conclusion}

A counterfactual probe of visual grounding is capped by how perceptible its edit
is, and the missing quantity is identified rather than merely bounded by three
measurements the same audit already collects. The error runs one way, so the
characteristic failure is a false alarm: an auditor discarding an overseer that
reads the image. Pairing the counterfactual score with a detection probe on the
unedited image costs nothing and turns the shortfall into a diagnosis.

%%%%%%%%%%%%%%%%%%%%%%%%%%%%%%%%%%%%%%%%%%%%%%%%%%%%%%%%%%%%%%%%%%%%%%%%%%%%%%%%
%% BACK MATTER
%%%%%%%%%%%%%%%%%%%%%%%%%%%%%%%%%%%%%%%%%%%%%%%%%%%%%%%%%%%%%%%%%%%%%%%%%%%%%%%%
\begin{kilrepro}
Pipeline, edits, prompts, metric code, per-instance verdicts and licences are in
the Code and Data Supplement. Every interval and $p$-value here is recomputed
from those shipped verdicts by \texttt{run\_missing.py} rather than stored
separately, which we confirmed by re-running it against the released files.
Reproducibility gaps we could not close (unpinned checkpoint revisions, unlogged
API access dates and inference hours) are stated in full in \cref{sec:judges}
and~\cref{sec:checklist} rather than papered over.
\end{kilrepro}

% NOTE: the source paper is anonymized and carries no author-contributions
% statement. Left commented rather than invented.
% \begin{contributions}
% \end{contributions}

% NOTE: no acknowledgments section exists in the source paper.
% \begin{acknowledgments}
% \end{acknowledgments}

% NOTE: no funding statement exists in the source paper.
% \begin{funding}
% \end{funding}

\begin{availability}
% This section SURVIVES blind mode (most venues require it), so the URL is
% anonymized explicitly -- \ifbool{KIL@blind}{...}{...} does it automatically.
Pipeline, edits, prompts, metric code, per-instance verdicts and licences are
available at
\ifbool{KIL@blind}
  {\url{https://anonymous.4open.science/r/XXXXXX}}
  {\url{https://github.com/KurbanIntelligenceLab/groundJudge}}.
The study relies only on publicly available
datasets and models. Dataset licences are itemized in \cref{sec:checklist}.
\end{availability}

\begin{conflicts}
\emph{Ethics, impact, and reproducibility.} A team that reads a low score
and concludes its judge is unusable discards a usable overseer, which our six
certified false alarms make concrete; one that reads a moderate score and calls
the judge adequately grounded has bounded only the edits it did act on. Both
errors are invisible without the second measurement, and both propagate because
the same judge gates what reaches users and what trains the next model.
The diagnostic adds no generative capability and no dual-use concern.

\emph{Ethics and LLM-use disclosure.} Large language models were used
to assist with writing and editing this paper. The authors reviewed all
content and take full responsibility for the work. The study relies only
on publicly available datasets and models and raises no additional ethical
concerns.
\end{conflicts}

%%%%%%%%%%%%%%%%%%%%%%%%%%%%%%%%%%%%%%%%%%%%%%%%%%%%%%%%%%%%%%%%%%%%%%%%%%%%%%%%
\bibliography{references}

@article{zheng2023judging,
  title={Judging llm-as-a-judge with mt-bench and chatbot arena},
  author={Zheng, Lianmin and Chiang, Wei-Lin and Sheng, Ying and Zhuang, Siyuan and Wu, Zhanghao and Zhuang, Yonghao and Lin, Zi and Li, Zhuohan and Li, Dacheng and Xing, Eric and others},
  journal={Advances in neural information processing systems},
  volume={36},
  pages={46595--46623},
  year={2023}
}

@article{gu2024survey,
  title={A survey on llm-as-a-judge},
  author={Gu, Jiawei and Jiang, Xuhui and Shi, Zhichao and Tan, Hexiang and Zhai, Xuehao and Xu, Chengjin and Li, Wei and Shen, Yinghan and Ma, Shengjie and Liu, Honghao and others},
  journal={The Innovation},
  volume={7},
  number={6},
  year={2026},
  publisher={Elsevier}
}

@article{chen2024mllmjudge,
  title={Mllm-as-a-judge: Assessing multimodal llm-as-a-judge with vision-language benchmark},
  author={Chen, Dongping and Chen, Ruoxi and Zhang, Shilin and Liu, Yinuo and Wang, Yaochen and Zhou, Huichi and Zhang, Qihui and Wan, Yao and Zhou, Pan and Sun, Lichao},
  journal={arXiv preprint arXiv:2402.04788},
  year={2024}
}

@inproceedings{li2025vlrewardbench,
  title={Vl-rewardbench: A challenging benchmark for vision-language generative reward models},
  author={Li, Lei and Wei, Yuancheng and Xie, Zhihui and Yang, Xuqing and Song, Yifan and Wang, Peiyi and An, Chenxin and Liu, Tianyu and Li, Sujian and Lin, Bill Yuchen and others},
  booktitle={2025 IEEE/CVF Conference on Computer Vision and Pattern Recognition (CVPR)},
  pages={24657--24668},
  year={2025},
  organization={IEEE}
}

@article{yasunaga2025mmrewardbench,
  title={Multimodal rewardbench: Holistic evaluation of reward models for vision language models},
  author={Yasunaga, Michihiro and Zettlemoyer, Luke and Ghazvininejad, Marjan},
  journal={arXiv preprint arXiv:2502.14191},
  year={2025}
}

@inproceedings{zou2026informativeness,
  title={When Vision-Language Models Judge Without Seeing: Exposing Informativeness Bias},
  author={Zou, Xiaohan and Sridhar, Roshan and Safarzadeh, Mohammadtaher and Roth, Dan},
  booktitle={Proceedings of the 64th Annual Meeting of the Association for Computational Linguistics (Volume 1: Long Papers)},
  pages={15417--15448},
  year={2026}
}

@article{mj1grounded,
  title={MJ1: Multimodal Judgment via Grounded Verification},
  author={Kumar, Bhavesh and Feng, Dylan and Tang, Leonard},
  journal={arXiv preprint arXiv:2603.07990},
  year={2026}
}

@article{vlmconformal,
  title={VLM Judges Can Rank but Cannot Score: Task-Dependent Uncertainty in Multimodal Evaluation},
  author={Kumar, Divake and Tayebati, Sina and Naik, Devashri and Krishnan, Ranganath and Trivedi, Amit Ranjan},
  journal={arXiv preprint arXiv:2604.25235},
  year={2026}
}

@article{illposed,
  title={Ill-Posed by Design: Probing Evidence Use in VLMs},
  author={Meivar, Boaz and Perek, Shaked and Shvartzman, Shani and Schwartz, Eli and Avidan, Shai},
  journal={arXiv preprint arXiv:2606.24335},
  year={2026}
}

@inproceedings{fooling,
  title={Fooling the LVLM judges: Visual biases in LVLM-based evaluation},
  author={Hwang, Yerin and Lee, Dongryeol and Min, Kyungmin and Kang, Taegwan and Kim, Yongil and Jung, Kyomin},
  booktitle={Proceedings of the 2025 Conference on Empirical Methods in Natural Language Processing},
  pages={23197--23216},
  year={2025}
}

@article{viable,
  title={A Visually Impaired Assistance Benchmark for VLM-as-a-Judge Evaluation},
  author={Zhao, Yi and Wang, Siqi and Hu, Zhe and Li, Yushi and Li, Jing},
  journal={arXiv preprint arXiv:2605.31351},
  year={2026}
}

@inproceedings{langbias,
  title={Looking beyond text: Reducing language bias in large vision-language models via multimodal dual-attention and soft-image guidance},
  author={Zhao, Haozhe and Si, Shuzheng and Chen, Liang and Zhang, Yichi and Sun, Maosong and Chang, Baobao and Zhang, Minjia},
  booktitle={Proceedings of the 2025 Conference on Empirical Methods in Natural Language Processing},
  pages={19677--19701},
  year={2025}
}

@inproceedings{debiasedselfjudge,
  title={Improving Alignment in LVLMs with Debiased Self-Judgment.},
  author={Yang, Sihan and Cui, Chenhang and Zhao, Zihao and Zhou, Yiyang and Yan, Weilong and Wei, Ying and Yao, Huaxiu},
  booktitle={EMNLP (Findings)},
  pages={8213--8232},
  year={2025}
}

@article{wang2025visualprm,
  title={Visualprm: An effective process reward model for multimodal reasoning},
  author={Wang, Weiyun and Gao, Zhangwei and Chen, Lianjie and Chen, Zhe and Zhu, Jinguo and Zhao, Xiangyu and Liu, Yangzhou and Cao, Yue and Ye, Shenglong and Zhu, Xizhou and others},
  journal={arXiv preprint arXiv:2503.10291},
  year={2025}
}

@article{evpv,
  title={Grounding the Score: Explicit Visual Premise Verification for Reliable Vision-Language Process Reward Models},
  author={Wang, Junxin and Guan, Dai and Qiu, Weijie and Li, Zhihang and Gai, Yongbo and Yang, Zhengyi and Zhou, Mengyu and Zhao, Erchao and Jiang, Xiaoxi and Jiang, Guanjun},
  journal={arXiv preprint arXiv:2603.16253},
  year={2026}
}

@article{vrprm,
  title={Vrprm: Process reward modeling via visual reasoning},
  author={Chen, Xinquan and Yue, Chongying and Liu, Bangwei and Wang, Xuhong and Wang, Yingchun and Lu, Chaochao},
  journal={arXiv preprint arXiv:2508.03556},
  year={2025}
}

@article{faithfulgrpo,
  title={Faithful grpo: Improving visual spatial reasoning in multimodal language models via constrained policy optimization},
  author={Kancheti, Sai Srinivas and Kanade, Aditya and Sinha, Rohit and Balasubramanian, Vineeth N and Ganu, Tanuja},
  journal={arXiv preprint arXiv:2604.08476},
  year={2026}
}

@article{treble,
  title={Treble counterfactual vlms: A causal approach to hallucination},
  author={Li, Li and Qu, Jiashu and Song, Linxin and Zhou, Yuxiao and Qin, Yuehan and Yang, Tiankai and Zhao, Yue},
  journal={arXiv preprint arXiv:2503.06169},
  year={2025}
}

@article{countercount,
  title={CounterCount: A Diagnostic Framework for Counting Bias in Vision Language Models},
  author={Alzahrani, Reem and Alshanqiti, Hassan and Hemid, Bushra Bin and Alyafeai, Zaid and Eldesokey, Abdelrahman and Ghanem, Bernard},
  journal={arXiv preprint arXiv:2605.17826},
  year={2026}
}

@inproceedings{hudson2019gqa,
  title={Gqa: A new dataset for real-world visual reasoning and compositional question answering},
  author={Hudson, Drew A and Manning, Christopher D},
  booktitle={2019 IEEE/CVF Conference on Computer Vision and Pattern Recognition (CVPR)},
  pages={6693--6702},
  year={2019},
  organization={IEEE}
}

@inproceedings{masry2022chartqa,
  title={Chartqa: A benchmark for question answering about charts with visual and logical reasoning},
  author={Masry, Ahmed and Tan, Jia Qing and Joty, Shafiq and Hoque, Enamul and others},
  booktitle={Findings of the association for computational linguistics: ACL 2022},
  pages={2263--2279},
  year={2022}
}

@inproceedings{lu2024mathvista,
  title={Mathvista: Evaluating mathematical reasoning of foundation models in visual contexts},
  author={Lu, Pan and Bansal, Hritik and Xia, Tony and Liu, Jiacheng and Li, Chunyuan and Hajishirzi, Hannaneh and Cheng, Hao and Chang, Kai-Wei and Galley, Michel and Gao, Jianfeng},
  booktitle={International Conference on Learning Representations},
  volume={2024},
  pages={23439--23554},
  year={2024}
}

@article{lu2022scienceqa,
  title={Learn to explain: Multimodal reasoning via thought chains for science question answering},
  author={Lu, Pan and Mishra, Swaroop and Xia, Tanglin and Qiu, Liang and Chang, Kai-Wei and Zhu, Song-Chun and Tafjord, Oyvind and Clark, Peter and Kalyan, Ashwin},
  journal={Advances in neural information processing systems},
  volume={35},
  pages={2507--2521},
  year={2022}
}

@inproceedings{guan2024hallusionbench,
  title={Hallusionbench: an advanced diagnostic suite for entangled language hallucination and visual illusion in large vision-language models},
  author={Guan, Tianrui and Liu, Fuxiao and Wu, Xiyang and Xian, Ruiqi and Li, Zongxia and Liu, Xiaoyu and Wang, Xijun and Chen, Lichang and Huang, Furong and Yacoob, Yaser and others},
  booktitle={2024 IEEE/CVF Conference on Computer Vision and Pattern Recognition (CVPR)},
  pages={14375--14385},
  year={2024},
  organization={IEEE}
}

@article{figureqa,
  title={Figureqa: An annotated figure dataset for visual reasoning},
  author={Kahou, Samira Ebrahimi and Michalski, Vincent and Atkinson, Adam and K{\'a}d{\'a}r, {\'A}kos and Trischler, Adam and Bengio, Yoshua},
  journal={arXiv preprint arXiv:1710.07300},
  year={2017}
}

@inproceedings{mmrb2,
  title={Multimodal rewardbench 2: Evaluating omni reward models for interleaved text and image},
  author={Hu, Yushi and Askari-Hemmat, Reyhane and Hall, Melissa and Dinan, Emily and Zettlemoyer, Luke and Ghazvininejad, Marjan},
  booktitle={Proceedings of the IEEE/CVF Conference on Computer Vision and Pattern Recognition},
  pages={36904--36915},
  year={2026}
}

@article{ppjd,
  title={Mitigating Perceptual Judgment Bias in Multimodal LLM-as-a-Judge via Perceptual Perturbation and Reward Modeling},
  author={Park, Seojeong and Choi, Jiho and Kang, Junyong and Lee, Seonho and Shin, Jaeyo and Shim, Hyunjung},
  journal={arXiv preprint arXiv:2606.02578},
  year={2026}
}

@article{vlmbiased,
  title={Vision language models are biased},
  author={Vo, An and Nguyen, Khai-Nguyen and Taesiri, Mohammad Reza and Dang, Vy Tuong and Nguyen, Anh Totti and Kim, Daeyoung},
  journal={arXiv preprint arXiv:2505.23941},
  year={2025}
}

@article{arbitration,
  title={Arbitration Failure, Not Perceptual Blindness: How Vision-Language Models Resolve Visual-Linguistic Conflicts},
  author={Nooralahzadeh, Farhad and Rohanian, Omid and Zhang, Yi and F{\"u}rst, Jonathan and Stockinger, Kurt},
  journal={arXiv preprint arXiv:2604.09364},
  year={2026}
}

@article{seeingsaying,
  title={Are vlms seeing or just saying? uncovering the illusion of visual re-examination},
  author={Shi, Chufan and Yang, Cheng and Wu, Yaokang and Jin, Linghao and Shui, Bo and Berg-Kirkpatrick, Taylor and Ma, Xuezhe},
  journal={arXiv preprint arXiv:2605.15864},
  year={2026}
}

@article{medcf,
  title={Do Medical Vision Language Models Actually See? A Counterfactual Grounding Framework and Hard-Negative Contrastive Training for Visually-Reliant Medical VLMs},
  author={Zafar, Anas and Murali, Leema Krishna and Bharadwaj, Siddhant and Vashist, Ashish and Wu, Jia},
  journal={arXiv preprint arXiv:2607.03647},
  year={2026}
}

@inproceedings{gao2023overopt,
  title={Scaling laws for reward model overoptimization},
  author={Gao, Leo and Schulman, John and Hilton, Jacob},
  booktitle={International conference on machine learning},
  pages={10835--10866},
  year={2023},
  organization={PMLR}
}

@article{ouyang2022instruct,
  title={Training language models to follow instructions with human feedback},
  author={Ouyang, Long and Wu, Jeffrey and Jiang, Xu and Almeida, Diogo and Wainwright, Carroll and Mishkin, Pamela and Zhang, Chong and Agarwal, Sandhini and Slama, Katarina and Ray, Alex and others},
  journal={Advances in neural information processing systems},
  volume={35},
  pages={27730--27744},
  year={2022}
}

@article{skalse2022hacking,
  title={Defining and characterizing reward gaming},
  author={Skalse, Joar and Howe, Nikolaus and Krasheninnikov, Dmitrii and Krueger, David},
  journal={Advances in neural information processing systems},
  volume={35},
  pages={9460--9471},
  year={2022}
}

@article{rewardattack,
  title={Reward under attack: Analyzing the robustness and hackability of process reward models},
  author={Tiwari, Rishabh and Tomar, Aditya and Bamba, Udbhav and Maheswaran, Monishwaran and Yang, Heng and Mahoney, Michael W and Keutzer, Kurt and Gholami, Amir},
  journal={arXiv preprint arXiv:2603.06621},
  year={2026}
}

@article{discreterm,
  title={Discretizing Reward Models},
  author={Viswanathan, Vijay and Wang, Shiqi and Hazarika, Devamanyu and Nagpal, Chirag and Wu, Tongshuang and Neubig, Graham and Mao, Yuning},
  journal={arXiv preprint arXiv:2606.21795},
  year={2026}
}

@article{rewardauditor,
  title={Reward auditor: Inference on reward modeling suitability in real-world perturbed scenarios},
  author={Zang, Jianxiang and Wei, Yongda and Bai, Ruxue and Jiang, Shiyu and Mo, Nijia and Li, Binhong and Sun, Qiang and Liu, Hui},
  journal={arXiv preprint arXiv:2512.00920},
  year={2025}
}

@article{kenton2024oversight,
  title={On scalable oversight with weak llms judging strong llms},
  author={Kenton, Zachary and Siegel, Noah Y and Kram{\'a}r, J{\'a}nos and Brown-Cohen, Jonah and Albanie, Samuel and Bulian, Jannis and Agarwal, Rishabh and Lindner, David and Tang, Yunhao and Goodman, Noah and others},
  journal={Advances in Neural Information Processing Systems},
  volume={37},
  pages={75229--75276},
  year={2024}
}

%%%%%%%%%%%%%%%%%%%%%%%%%%%%%%%%%%%%%%%%%%%%%%%%%%%%%%%%%%%%%%%%%%%%%%%%%%%%%%%%
\appendix

% The supplement's opening note, carried over. Its numbering references are
% rewritten as \cref to the merged document's own labels.
\noindent The appendix supplies the structural causal model and the supporting
results (\cref{sec:scm}), the proofs (\cref{sec:proofs}), the experimental
detail (\cref{sec:expdetail}), and the reproducibility notes and full grids
(\cref{sec:checklist}). The main paper states, in order: \cref{def:vgs} the
Verdict Grounding Score, \cref{ass:med} consistency judging, \cref{lem:perc} the
reading bound, \cref{ass:match} matched and separable attributes,
\cref{thm:bound} the identification of the edit-miss rate, \cref{cor:order} the
ordering and what a violation certifies, \cref{cor:onesided} the one-sided error
and the false-alarm certificate, and \cref{prop:gap} the image-blindness
certificate. Proofs of all of these are in \cref{sec:proofs}. Five supporting
results that the main paper carries in prose rather than as displayed statements
are stated and proved in \cref{sec:support}, and Proposition~A.0 records
identification by construction.

\section{Structural causal model, identification, estimation}
\label{sec:scm}

\paragraph{Structural causal model and identification}
Let an instance be generated by exogenous $U$ with $(X,R)=f_{XR}(U)$, and let
the ground truth $Y=y(X,R)$ be a deterministic function of the pair. A judge
is $V=J(X,R;U_J)$ with judge randomness $U_J$ independent of $U$. For a value
$x'$, the intervention $\dof(X=x')$ replaces the image argument and induces
the interventional law $P\big(J(x',R;U_J)\mid R\big)$. The controlled direct
effect of the image on the verdict, holding $R$ fixed, contrasts two such
interventions. $\vgs$ is the population contrast between the truth-flipping
intervention $\dof(X=\tilde X)$ and the truth-preserving control
$\dof(X=X^{c})$, each taken relative to $\dof(X=X)$.

\begin{supprop}{Proposition A.0 (Identification by construction)}
Because the diagnostic realizes $\tilde X$ and $X^{c}$ as constructed images
and queries $J$ on them, the laws $P(J(\tilde X,R))$, $P(J(X^{c},R))$ and
$P(J(X,R))$ are sampled directly. Hence $\vgs$, $\gedr$ and $\ledr$ are
point-identified from the collected verdicts without any unconfoundedness or
back-door assumption. The requirements are edit validity
($y(\tilde X,R)=\neg y(X,R)$, $y(X^{c},R)=y(X,R)$) and i.i.d.\ sampling of
instances.
\end{supprop}

\begin{proof}
Each quantity is a functional of the three interventional laws above.
Physical construction of the intervened images means the observed verdicts
are i.i.d.\ draws from exactly those laws, so the plug-in functionals equal
the estimands in expectation. Edit validity is what makes
$\Delta_{\mathrm{flip}}$ a response to a truth change and
$\Delta_{\mathrm{null}}$ a response to none.
\end{proof}

Edit validity is an empirical premise, not a theorem. \Cref{sec:expdetail}
reports the validation we ran for it, and \cref{sec:gqa} reports the dataset on
which it failed.

\subsection{Supporting results, stated and proved}
\label{sec:support}

The five statements below (A.1 to A.5; A.0 above is the identification
argument) are routine and are not the paper's contribution, so the main text
carries them in prose and points here for the statement and the proof.

\begin{supprop}{Proposition A.1 (Range and anchors)}
$\vgs(J)\in[-1,1]$ for every judge. If $J$ is image-blind ($J(x,r)=g(r)$ for
some $g$ and all $x$) then $\vgs(J)=0$. If $J$ is grounded ($J(X,R)=y(X,R)$
a.s.) and the edits are valid, then $\vgs(J)=1$ for binary or pairwise verdicts.
\end{supprop}

\begin{proof}
$\vgs$ is a difference of two probabilities, hence in $[-1,1]$. If $J(x,r)=g(r)$
then $J(\tilde X,R)=J(X^{c},R)=J(X,R)$ pointwise, so both terms vanish. If
$J=y$, edit validity gives $y(\tilde X,R)=\neg y(X,R)$, so
$\Delta_{\mathrm{flip}}=1$, while $y(X^{c},R)=y(X,R)$ gives
$\Delta_{\mathrm{null}}=0$.
\end{proof}

\begin{supprop}{Proposition A.2 (Accuracy does not bound grounding)}
For every $a\in[0,1]$ there exist a benchmark distribution and a judge with
preference accuracy $a$ and $\vgs=0$, and all four combinations of high/low
accuracy and high/low $\vgs$ are realizable.
\end{supprop}

\begin{proof}
Take any image-blind $J(x,r)=g(r)$; by Proposition~A.1
$\vgs(J)=0$ whatever $g$ is, and choosing $g$ to agree with $Y$ on a fraction
$a$ of instances realizes accuracy $a$. For the remaining combinations, $J=y$
gives accuracy $1$ with $\vgs=1$, and $J=\neg y$ gives accuracy $0$ with
$\vgs=1$.
\end{proof}

\begin{supprop}{Proposition A.3 (Image sensitivity caps the accuracy the image can buy)}
Write $a(J)=\Pr[J(X,R)=Y]$ and let $J_{x_0}(\cdot,r)=J(x_0,r)$ freeze the image
at a reference $x_0$. Then
$\lvert a(J)-a(J_{x_0})\rvert\le \Pr[J(X,R)\neq J(x_0,R)]$.
\end{supprop}

\begin{proof}
$a(J)-a(J_{x_0})=\E[\mathbf 1\{J(X,R)=Y\}-\mathbf 1\{J(x_0,R)=Y\}]$, and
the integrand is $0$ on the event $\{J(X,R)=J(x_0,R)\}$ and bounded by $1$ in
absolute value elsewhere.
\end{proof}

\begin{supprop}{Proposition A.4 (Which image each probe is a property of)}
Consider a comparison of the audited form, a clean trace against one carrying an
injected visual error, scored on the \emph{unedited} image. Under
\cref{ass:med} and \cref{ass:match}(a) its success rate is $\pi_1=\gedr$, which
contains no term
in $\pi_2$. By \cref{thm:bound}, $\vgs=\pi_2-\Delta_{\mathrm{null}}$ instead
depends on
the judge's reading of the edited image $\tilde X$.
\end{supprop}

\begin{proof}
The first claim is the first equality of \cref{thm:bound} applied to that
comparison;
the second is \cref{cor:onesided}. Neither involves the other's argument.
\end{proof}

\begin{supprop}{Corollary A.5 (Estimation)}
The plug-in estimators of $\vgs$, $\gedr$ and $\ledr$ over $n$ i.i.d.\
instances are unbiased, and
$n\ge\tfrac{2}{\varepsilon^{2}}\log\tfrac{4}{\delta}$ suffices for
$\lvert\widehat{\vgs}-\vgs\rvert\le\varepsilon$ with probability at least
$1-\delta$.
\end{supprop}

\begin{proof}
Each estimator is a sample mean of a bounded indicator, hence unbiased.
$\widehat{\vgs}$ is a difference of two such means; applying Hoeffding at level
$\varepsilon/2$ to each and a union bound gives failure probability at most
$4\exp(-n\varepsilon^{2}/2)$, which is at most $\delta$ for the stated $n$.
\end{proof}

\subsection{General detection-gap decomposition}
\label{sec:decomp}

Let $b(J)=\tfrac{1-p_{\mathrm{tie}}}{2}\le\tfrac12$ be the judge's
visual-detection rate with the image ablated, estimable by running the judge
under $x_\varnothing$. Write $\gedr=b(J)+g_{\mathrm{img}}(J)$, splitting
detection into a text-channel baseline and an image-driven excess. By
\cref{prop:gap} an image-blind judge has $g_{\mathrm{img}}=0$, so all detection
above $b(J)$ is attributable to reading the image. Logical corruption is
text-internal, so $\ledr$ is available from the text channel; hence
$\ledr-\gedr=(\ledr-b(J))-g_{\mathrm{img}}(J)$.

This decomposition is only as good as \cref{ass:med}, and \cref{sec:sciqa}
reports a
dataset on which our own text-only checks reached $\gedr=0.64$, above the
$\tfrac12$ ceiling the assumption implies. $b(J)$ was measured directly for
all 5 open-weight judges on all 4 datasets under $x_\varnothing$
(text-only): per-dataset point-estimate ranges are ChartQA $0.00$--$0.51$,
MathVista $0.00$--$0.50$, HallusionBench $0.25$--$0.60$, ScienceQA
$0.36$--$0.64$, and none of the $20$ cells' $95\%$ CI excludes $0.5$ (full
result and its implications in \cref{sec:audit} and \cref{sec:sciqa}).
$\Delta_{\mathrm{null}}$ is tabulated per cell
in \cref{stab:main}, from data already collected with no new judge calls,
converting the shortfall from identified-up-to-$\Delta_{\mathrm{null}}$ into
a point estimate throughout. The read-attribute experiment measures the
edit-responsive reading rate
$\mathrm{ERR}$ of \cref{lem:perc} directly: each open-weight judge reports the
attribute value under $x$ and under $\tilde x$ via \texttt{read\_attribute()},
recording whether it moves, which converts $\pi_1-\pi_2$ from an inferred
residual into an observed rate (\cref{sec:readattr} reports the result and the
three mechanisms it surfaces).

\paragraph{What \texorpdfstring{$\beta$}{beta} is and is not}
$\beta(x)$ is the value the
verdict actually uses, not the content of the visual encoder. \Cref{lem:perc}
therefore
holds whether the edit is never perceived or is perceived, attended, and then
outweighed by a language prior downstream. Both leave $\beta$ unchanged, and
recent work argues for each: side-by-side detection failure supports the first,
while attention-level analyses that find the region attended and the evidence
encoded support the second. We do not adjudicate, and the bound does not
require us to.

\subsection{Artifact-matched control}
\label{sec:artifact}

Model the flip probability under an edit operator $E$ as $\pi_E=c(E)+\alpha$,
where $c(E)$ is the content-induced component and $\alpha\ge0$ is an
artifact-induced component that depends on the operator's footprint but not on
whether content changed. Under this additive, content-independent artifact
model, the null edit $X^{\dagger}$ applies the same operator to a
verdict-irrelevant region, so $c=0$ and $\pi_{X^{\dagger}}=\alpha$, while the
truth-flipping edit has $\pi_{\tilde X}=c_{\mathrm{flip}}+\alpha$. Using
$X^{c}=X^{\dagger}$ in \cref{def:vgs} yields
$\vgs=(c_{\mathrm{flip}}+\alpha)-\alpha=c_{\mathrm{flip}}$, netting out both
re-encoding noise and edit artifacts. The additive model is an idealization and
breaks if artifact
sensitivity interacts with content, for instance if a
judge only notices the operator's footprint when it lands on a
verdict-relevant region.

$X^{\dagger}$ records exist for the eight original judges on all four datasets
(2,129 records); none were collected for Gemini-3.5-Flash. Recomputing $\vgs$
with $X^{c}=X^{\dagger}$ for those $32$ cells and comparing to the
$X^{+}$-based value, the $10$ of $32$ cells where $|{\rm diff}|\ge0.01$ are:

\begin{table}[t]
\centering
\caption{\textbf{Artifact-matched control.} The ten of $32$ cells where the
$X^{\dagger}$-based and $X^{+}$-based $\vgs$ differ by at least $0.01$; the
other $22$ match to two decimals.}
\label{tab:artifact}
\small
\begin{tabular}{@{}ll S[table-format=-1.2] S[table-format=-1.2] S[table-format=-1.2]@{}}
\toprule
\kilth{Judge} & \kilth{Dataset} & {\kilth{$\vgs(X^{+})$}} & {\kilth{$\vgs(X^{\dagger})$}} & {\kilth{diff}} \\
\midrule
LLaVA-OneVision & ChartQA & 0.21 & 0.20 & -0.01 \\
Qwen2.5-VL-7B & ChartQA & 0.49 & 0.41 & -0.08 \\
InternVL3-8B & ChartQA & 0.29 & 0.27 & -0.02 \\
InternVL3-8B & MathVista & 0.06 & 0.19 & 0.12 \\
Llama-3.2-Vision & ChartQA & -0.05 & 0.01 & 0.06 \\
GPT-4o & MathVista & 0.12 & 0.19 & 0.06 \\
Gemini-2.5-Flash & ChartQA & 0.38 & 0.27 & -0.11 \\
Claude-Sonnet-4.5 & MathVista & 0.00 & 0.06 & 0.06 \\
VisualPRM-8B-v1.1 & ChartQA & 0.54 & 0.46 & -0.08 \\
VisualPRM-8B-v1.1 & ScienceQA & 0.59 & 0.68 & 0.09 \\
\bottomrule
\end{tabular}
\end{table}

Mean absolute difference across all $32$ cells: $0.025$;
maximum $0.125$ (InternVL3-8B/MathVista). Edit-artifact sensitivity is a
small effect relative to the content effect $\vgs$ measures under either
control, so the $X^{+}$-based numbers reported throughout are not
meaningfully inflated by artifacts, consistent with the additive model
above holding reasonably well in practice.

\subsection{Graded-verdict extension}
\label{sec:graded}

For a scalar judge $s=J(x,r)\in\R$, let
$d_{\mathrm{flip}}=s(\tilde X,R)-s(X,R)$ and
$d_{\mathrm{null}}=s(X^{c},R)-s(X,R)$, and define
$\vgs_{s}=\E[\sigma\,d_{\mathrm{flip}}]-\E[|d_{\mathrm{null}}|]$,
where $\sigma=+1$ if the flip raises the ground-truth target and $-1$
otherwise, with scores min-max normalized per judge. Thresholding $s$ at the
judge's decision boundary recovers the binary verdict and reduces $\vgs_s$ to
\cref{def:vgs}, so the graded score generalizes it. The anchors of
Proposition~A.1
carry over: $\vgs_s$ attains its maximum for a grounded scorer whose target
equals $y$, and $\vgs_s=0$ for an image-blind scorer. The graded score
inherits the judge's calibration, which is why the main paper reports the
binary form.

%%%%%%%%%%%%%%%%%%%%%%%%%%%%%%%%%%%%%%%%%%%%%%%%%%%%%%%%%%%%%%%%%%%%%%%%%%%%%%%%
\section{Proofs}
\label{sec:proofs}

Proofs appear in the order the main paper states the results.

\begin{proof}[Proof of \cref{cor:onesided}]
Under \cref{ass:med} and \cref{ass:match}(b), $\Delta_{\mathrm{flip}}=\pi_2$
(proof of \cref{thm:bound}), so
$\vgs=\Delta_{\mathrm{flip}}-\Delta_{\mathrm{null}}=\pi_2-\Delta_{\mathrm{null}}$.
Under \cref{ass:med} and \cref{ass:match}(a), $\gedr=\pi_1$
(\cref{thm:bound}). Hence
$\gedr-\vgs=(\pi_1-\pi_2)+\Delta_{\mathrm{null}}$. If $\pi_2\le\pi_1$ then both
summands are non-negative, giving
$\vgs\le\gedr-\Delta_{\mathrm{null}}\le\gedr$, and the deficit
$\gedr-\vgs$ equals $(\pi_1-\pi_2)+\Delta_{\mathrm{null}}$ exactly. The
certificate is \cref{prop:gap} applied to the same cell: $\gedr>b$ excludes an
image-blind judge, so a low $\vgs$ in that cell cannot be attributed to the
judge ignoring the image, and by the display above the whole of $\gedr-\vgs$ is
accounted for by missed edits and oversensitivity. Setting
$\Delta_{\mathrm{null}}=0$ recovers $\vgs=\pi_2$ as the specific-judge case.
\end{proof}

\begin{proof}[Proof of \cref{lem:perc}]
Under \cref{ass:med} the verdict is
$J(x,r)=\varphi_r(\mathbf 1\{\beta(x)=a(r)\})$. On the event
$\{\beta(\tilde X)=\beta(X)\}$ the indicator takes the same value under both
images, so $J(\tilde X,R)$ and $J(X,R)$ are the same $\varphi_R$ evaluated at
the same argument and agree. Hence
$\{J(\tilde X,R)\neq J(X,R)\}\subseteq\{\beta(\tilde X)\neq\beta(X)\}$,
$\Delta_{\mathrm{flip}}\le\mathrm{ERR}$, and $\vgs\le\mathrm{ERR}$ because
$\Delta_{\mathrm{null}}\ge0$. No assumption beyond \cref{ass:med} is used, and in
particular nothing is assumed about \emph{why} the reading fails to move.
\end{proof}

\subsection{Two facts that hold without Assumption 2}
\label{sec:twofacts}

Both are used repeatedly below. Recall $a_T=v^{\star}(X)$, $a_v\neq a_T$ the
injected value, and $v^{\star}(\tilde X)=a_v$ by edit validity.

\emph{(i) $\gedr=\pi_1$, using \cref{ass:match}(a).} Under \cref{ass:med},
$J(X,T)=\varphi_T(\mathbf 1\{\beta(X)=a_T\})$ and
$J(X,T_v)=\varphi_{T_v}(\mathbf 1\{\beta(X)=a_v\})$. \Cref{ass:match}(a) makes $T$
and $T_v$ exchangeable in the text channel, so the two traces share a scoring
map, $\varphi_T=\varphi_{T_v}=:\varphi$ with $\varphi$ strictly increasing.
Then $J(X,T)\succ J(X,T_v)$ exactly when
$\mathbf 1\{\beta(X)=a_T\}>\mathbf 1\{\beta(X)=a_v\}$, and because
$a_T\neq a_v$ at most one indicator can be $1$, so the event reduces to
$\{\beta(X)=a_T\}$ and $\gedr=\pi_1$. Without (a) the two traces could carry
different baselines and the reduction fails, which is why (a) is load-bearing
here and not only in \cref{prop:gap}.

\emph{(ii) $\Delta_{\mathrm{flip}}$ is a function of one indicator.} The trace
$T_v$ is held fixed across the two images, so by \cref{ass:med} the verdict at
image $z$ is $\varphi_{T_v}(m(z))$ with
$m(z)=\mathbf 1\{\beta(z)=a_v\}$; since $\varphi_{T_v}$ is strictly increasing
it is injective on $\{0,1\}$, so the verdict changes exactly when $m$ changes
and $\Delta_{\mathrm{flip}}=\Pr[m(\tilde X)\neq m(X)]$. \Cref{ass:med} also fixes
the judge's randomness across image conditions; without that, both
$\Delta_{\mathrm{flip}}$ and $\Delta_{\mathrm{null}}$ pick up the same
intrinsic instability term, which cancels in $\vgs$ but leaves
$\Delta_{\mathrm{flip}}=\pi_2$ holding only up to it. This is the one place
where a judge queried at nonzero temperature without common random numbers
would break the identity rather than merely widen it.

\begin{proof}[Proof of \cref{thm:bound}]
By \cref{sec:twofacts}(ii), writing $q=\Pr[m(X)=1]=\Pr[\beta(X)=a_v]$ and noting
$\Pr[m(\tilde X)=1]=\Pr[\beta(\tilde X)=v^{\star}(\tilde X)]=\pi_2$,
\begin{equation}\label{eq:gen}
\Delta_{\mathrm{flip}}=q+\pi_2-2c,\qquad
c=\Pr[m(X)=1,\ m(\tilde X)=1].
\end{equation}
\Cref{ass:match}(b) sets $q=0$, which forces $c=0$ because $c\le q$, so
$\Delta_{\mathrm{flip}}=\pi_2$. With \cref{sec:twofacts}(i) this gives
$\vgs=\Delta_{\mathrm{flip}}-\Delta_{\mathrm{null}}=\pi_2-\Delta_{\mathrm{null}}$
and $\gedr=\pi_1$, hence
$\gedr-\vgs=\pi_1-\pi_2+\Delta_{\mathrm{null}}$, which rearranges to the
statement. Every term on the right is measured by the protocol, so
$\pi_1-\pi_2$ is point-identified. Nothing is assumed about the sign of
$\pi_1-\pi_2$.
\end{proof}

\begin{proof}[Proof of \cref{cor:order}]
Under the hypotheses of \cref{thm:bound},
$\gedr-\vgs=(\pi_1-\pi_2)+\Delta_{\mathrm{null}}$, so $\gedr-\vgs\ge0$ if and
only if $\pi_2-\pi_1\le\Delta_{\mathrm{null}}$. Since
$\Delta_{\mathrm{null}}\ge0$, $\pi_2\le\pi_1$ is sufficient. Conversely
$\vgs>\gedr$ forces $\pi_2>\pi_1+\Delta_{\mathrm{null}}\ge\pi_1$, so a
violation certifies that the edit strictly raised the judge's read accuracy on
the attribute.
\end{proof}

\subsection{Why the ordering is not definitional}
\label{sec:notdef}

$\vgs$ and $\gedr$ are functionals of different interventions and are not
constrained by definition. Let $J$ flip its verdict whenever any two images
differ and let $h$ assign the same score to every trace. Then
$\Delta_{\mathrm{flip}}=1$, $\Delta_{\mathrm{null}}=1$ if the re-encoding is
detectable and $0$ otherwise, and $\gedr=0$ because the judge never separates
$T$ from $T_v$. Taking a re-encoding the judge cannot detect gives $\vgs=1>0=
\gedr$. \Cref{cor:order} is therefore a falsifiable prediction on real judges
rather than an algebraic consequence.

\subsection{The identity without Assumption 2(b)}
\label{sec:nob}

Dropping \cref{ass:match}(b) and keeping \cref{eq:gen},
\[
\gedr-\vgs=\pi_1-q-\pi_2+2c+\Delta_{\mathrm{null}},
\]
so the shortfall identifies $\pi_1-\pi_2-q+2c$ rather than $\pi_1-\pi_2$. The
binary-attribute case is the specialization $q=1-\pi_1$, since a misread of a
two-valued attribute necessarily lands on $a_v$:
\[
\gedr-\vgs=(2\pi_1-1)-\pi_2+2c+\Delta_{\mathrm{null}},
\]
and the ordering holds if and only if
$2\pi_1-1\ge\pi_2-2c-\Delta_{\mathrm{null}}$,
which is not implied by $\pi_2\le\pi_1$. HallusionBench and MathVista are
binary in this sense. We therefore report them separately in the main paper and
treat the ordering there as an observation rather than a confirmed prediction.
A judge below chance on the original attribute ($\pi_1<\tfrac12$) can violate
the ordering on a binary construction without the edit having become easier to
read, which is a second reading of the single MathVista exception.

\begin{proof}[Proof of \cref{prop:gap}]
An image-blind judge scores $T$ and $T_v$ from their text alone. Under
\cref{ass:match}(a) the pair is exchangeable in the text channel, so
$\Pr[J(T)\succ J(T_v)]=\Pr[J(T_v)\succ J(T)]=:q_0$, and
$2q_0+p_{\mathrm{tie}}=1$ gives $\gedr=q_0\le\tfrac12$. The contrapositive is
the certificate: $\gedr>\tfrac12$ implies the judge is not image-blind, after
which \cref{thm:bound} attributes the shortfall to $\pi_1-\pi_2$ and
$\Delta_{\mathrm{null}}$. \Cref{ass:match}(a) is the only part of
\cref{ass:match} this
proposition uses, and it is the part our own construction-gate evidence does
not fully establish on ScienceQA (\cref{sec:sciqa}).
\end{proof}

%%%%%%%%%%%%%%%%%%%%%%%%%%%%%%%%%%%%%%%%%%%%%%%%%%%%%%%%%%%%%%%%%%%%%%%%%%%%%%%%
\section{Experimental details}
\label{sec:expdetail}

\subsection{Datasets and ground truth}
\label{sec:datasets}

GroundJudge is constructed from datasets that expose structured visual ground
truth: GQA scene graphs, ChartQA source tables, MathVista figures, and
ScienceQA together with HallusionBench for science and hallucination-sensitive
cases. Because the visual ground truth is machine-readable, error injection
and counterfactual edits are automatic and exact, and no human annotation is
required. A human spot-check validates edit correctness on ChartQA and
HallusionBench ($n{=}50$ each, $100\%$ valid); MathVista and ScienceQA are
instead validated by full automated ground-truth reconstruction on every kept
instance. The main audit, the selection experiment and the reward-correlation
analysis all run on ChartQA ($n{=}100$), HallusionBench ($n{=}104$),
MathVista ($n{=}16$) and ScienceQA ($n{=}22$).

\subsection{MathVista: construction and scope}
\label{sec:mathvista}

MathVista ships no per-figure structured annotation and its \texttt{context}
metadata field is unreliable: instances tagged as pie charts or line plots
were found whose rendered image is a bar chart. We restrict to the
FigureQA-sourced subset of \texttt{testmini} with a yes/no comparison answer
($62$ of $1{,}000$ instances), which is the only subset admitting a
single-bar-size edit that flips the ground truth without other visible
inconsistencies. We re-derive chart geometry from pixels by connected-component
detection of each bar's rectangle, separating bars from axis lines and text by
shape rather than color, since some categories are literally named and colored
``Black.'' Each bar's category name is read by OCR under several rotations,
because label orientation is inconsistent across images. Every candidate
instance's predicted answer is recomputed from the detected bar sizes and
compared against the dataset's stated answer before the instance is kept:
$24$ instances pass bar detection, $17$ have their named bars identified, and
$16$ reproduce the stated ground truth exactly. One instance, a high-median
question on a six-bar chart, could not be reconciled under any tried median
convention and was dropped rather than guessed at. Each kept instance's
counterfactual edit is verified by resimulating the same predicate on the
post-edit bar sizes and confirming that it flips.

\subsection{ScienceQA: construction, and two leakage fixes}
\label{sec:sciqa}

We restrict to the ``State capitals'' category: a standardized
$750{\times}429$px US map with exactly one state filled in a highlight color
against a uniform base color, asking that state's capital. Identifying the
highlighted state is free because the question names it, so the only
re-derivation needed is locating the highlighted region and identifying a
different state's region to recolor into. We build a state-name to
region-centroid table from the dataset's own instances, match a candidate
replacement region by nearest centroid, and verify through a static
state-to-capital dictionary that the replacement's capital differs from the
true answer. Two of the original $50$ unique questions (Delaware and Rhode
Island, the two smallest states) fail highlight-region detection and are
dropped.

Two rounds of construction bugs were found through the main audit itself
rather than in advance. An initial trace design named the highlighted state in
text (``the
highlighted state is Wyoming\dots its capital is Denver''), which let a
text-only judge catch the visual error from US-geography knowledge alone. That
appeared as $\gedr\approx1$ for text-only judges in a first audit pass, where
\cref{prop:gap} says the rate should be at chance. Removing the state name from
the trace did not fix it, because the original question, passed to every judge
as context, also names the state; genericizing that reference still left
$\gedr\approx1$. The remaining cause was that the counterfactual's replacement
capital, chosen from any other state, was almost never among the question's
four multiple-choice options ($2/48$ by chance), and was therefore eliminable
by any judge as ``not one of the choices,'' independent of geography or the
image. Constraining the replacement to states whose capital is one of the
question's own wrong choices closed the leak. On the corrected construction,
text-only GEDR falls to $0.36$ to $0.64$ across judges and the instance count
drops from $48$ to $22$.

We report this history rather than omit it, and we draw the consequence the
main paper states: $0.64$ is above the $\tfrac12$ ceiling that \cref{ass:med}
implies, so the assumption is not fully restored by the fix. Text-only $\gedr$,
run
for all $5$ open-weight judges on all $4$ datasets with $2{,}000$-resample
bootstrap CIs, gives per-dataset ranges of ChartQA $0.00$--$0.51$, MathVista
$0.00$--$0.50$, HallusionBench $0.25$--$0.60$, ScienceQA $0.36$--$0.64$
(this dataset's range matches the one reported two paragraphs above
exactly). None of the $20$ cells' CIs exclude $0.5$, so no dataset meets the
withdrawal criterion originally proposed here --- see \cref{sec:audit} for the
full result and the point-estimate pattern worth noting anyway.

\subsection{GQA: an excluded negative result}
\label{sec:gqa}

GQA counterfactual edits use a deterministic HSV recolor of a scene-graph
object's region (SAM2 box-to-silhouette mask refinement, with hue separation
and saturation floors tuned against the region's dominant hue) after an
initial diffusion-inpainting construction plateaued near $60$ to $70\%$
human-validated edit rate. Seven rounds of fixes to the recolor pipeline
raised the rate from $55\%$ to a stable $70\%$ ($n{=}100$ human review), still
below the $95\%$ gate: mask blending in RGB rather than HSV space to remove a
halo artifact, bounding-box area filters, SAM2 mask refinement, exclusion of
low-texture background object names, a maximum-objects-per-scene filter
against visual redundancy, widened hue separation, and a raised minimum object
size for legibility at rendered scale. The remaining failures cluster into
three causes: recolors that are applied but not visually perceptible ($16/30$
failures), scenes containing an unlabeled visually-redundant duplicate of the
target object ($7/30$), and edits that do not read as the specifically named
object ($7/30$). We report this as a construction limitation specific to GQA's
real-world photographic diversity rather than a property of the audit, and
exclude GQA from every table rather than lower the bar or omit the attempt.

\subsection{Trace generation and error injection}
\label{sec:traces}

For each instance we obtain a multimodal reasoning trace, generated by a fixed
strong VLM or templated from the ground-truth annotations. We inject exactly
one visually-ungrounded step by selecting a ground-truth visual attribute
referenced in the trace and replacing it with a plausible false value drawn
from the dataset's attribute vocabulary (a chart value, an object count, a
color, an axis reading), while preserving grammaticality and downstream
textual consistency so the error cannot be detected from text alone. In a
separate condition we inject a logical error: a step that does not follow from
its predecessors, with all visual claims held true. The two conditions share
instances, so the logical-minus-visual contrast is within-instance.

\subsection{Image interventions}
\label{sec:interventions}

For each instance we form $x^{+}$, a content-preserving re-encoding that
preserves ground truth and measures intrinsic verdict instability;
$x_\varnothing$, a blank or noise image; $x'$, a mismatched image from a
different instance; and $\tilde x$, a counterfactual edit that makes the
injected claim true and flips the ground truth (\cref{sfig:interventions}
illustrates all four on one instance). The artifact-matched control
$x^{\dagger}$ of \cref{sec:artifact} was also constructed, for all 8 judges on
all 4
datasets (\cref{sec:artifact} reports the comparison against $x^+$).

\begin{figure}[t]
\centering
\kilgraphics[0.65\linewidth]{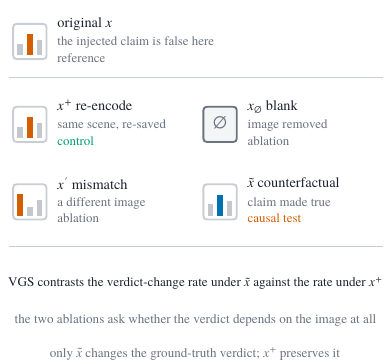}
\caption{\textbf{The four image interventions.} The reference cell is the
unedited image $x$, about which the injected claim is false; here the trace
claims that the tallest bar is blue and it is not. GroundJudge forms a
truth-preserving re-encoding $x^{+}$, a blank $x_\varnothing$ and a mismatched
$x'$, and a counterfactual $\tilde x$ edited so that the claim becomes true and
the ground-truth verdict flips.}
\label{sfig:interventions}
\end{figure}

\subsection{Judges}
\label{sec:judges}

Five open-weight judges run locally on a single GPU, loaded one at a time:
Qwen2.5-VL-7B-Instruct, InternVL3-8B (the natively transformers-integrated
checkpoint), LLaVA-OneVision-7B, Llama-3.2-11B-Vision-Instruct, and
VisualPRM-8B-v1.1, a trained multimodal process reward model scored through
its own step-level interface rather than free-text generation. Llama-3.2-Vision
sits behind a per-account license gate on Hugging Face that cleared after the
first round of runs; its results use the same code path, prompts and protocol
as the other open judges. Three proprietary judges (GPT-4o, Claude-Sonnet-4.5,
Gemini-2.5-Flash) are accessed through a single gateway routed to each
vendor's OpenAI-compatible endpoint rather than three separate vendor keys,
which changes neither prompts nor scoring. Two further process reward models
were considered and not run in this version.

Model identifiers used are exact: \texttt{Qwen/Qwen2.5-VL-7B-Instruct},
\texttt{OpenGVLab/InternVL3-8B-hf}, \texttt{llava-hf/llava-onevision-qwen2-7b-ov-hf},
\texttt{meta-llama/Llama-3.2-11B-Vision-Instruct},
\texttt{OpenGVLab/VisualPRM-8B-v1\_1}, \texttt{openai/gpt-4o},
\texttt{google/gemini-2.5-flash}, \texttt{anthropic/claude-
sonnet-4.5}, and
\texttt{google/gemini-3.5-flash} (the last accessed via OpenRouter). No
\texttt{revision} was pinned in \texttt{configs/judges.yaml} for any
open-weight checkpoint, and no API access date was logged anywhere in the
pipeline, so exact byte-for-byte reproduction cannot be guaranteed even with
the same code --- this is a real gap, not a reporting omission we are
covering up. GPU: all open-weight judges and this submission's new
read-attribute and magnitude-sweep experiments ran on single-GPU Slurm
allocations (\texttt{\#SBATCH --gres=gpu:1}, partition \texttt{gpu-A100},
i.e.\ one NVIDIA A100 per job, one judge loaded at a time). Judge-call
counts that are exactly countable from the released per-instance files: the
original construction-gate and main-audit collection (the paper's headline
\cref{tab:main} numbers, predating this session) is $26{,}712$ calls; the
read-attribute experiment (read\_attribute on 4 open-weight judges $\times$
4 datasets) is $1{,}936$ calls; the magnitude sweep (4 judges $\times$ 511
ChartQA magnitude instances) is $2{,}044$ calls. Total inference hours for
the original construction-gate and main-audit collection are not
recoverable from anything available to us now; this session's own
read-attribute and magnitude-sweep collection ran within the Slurm
wall-clock limits set in \texttt{panther\_jobs/run\_e4.sh} and
\texttt{run\_e6.sh} (24 hours each), but exact elapsed time was not logged
separately from queue wait time.

\subsection{Independent candidate pool for the selection experiment}
\label{sec:pool}

The selection experiment's original three-candidate pool (clean, visual-error,
logical-error) reuses traces already scored in the main audit, so a judge's
selection there could mirror its main-audit preference rather than reflect
best-of-$N$ behavior over a richer set. We build
additional visually-wrong candidates from each dataset's own ground-truth
vocabulary, with no separate generator model: for ChartQA, three further false
bar values at $+50\%$, $-50\%$ and $+25\%$ of the true value, clamped and
rounded as the original false value was; for ScienceQA, the two wrong
multiple-choice capitals not used as the counterfactual-edit target.
HallusionBench and MathVista are excluded from this extension because their
ground truth is a strict binary and the dataset admits exactly one possible
false claim per instance. Every additional candidate is scored against the
same unedited image $x$ the original pool uses, so a visually-wrong trace does
not need a matching edited image to exist. This required $8$ judges $\times$
($100$ ChartQA $+$ $22$ ScienceQA) instances $\times$ ($3$ or $2$) additional
candidates, or $2{,}752$ new judge calls.

\subsection{Metrics and statistics}
\label{sec:stats}

$\vgs$ follows \cref{def:vgs}. For pairwise judges the verdict is the
preference; for pointwise judges we threshold the score or use the score gap.
GEDR is the fraction of injected visual errors penalized, meaning the clean
trace is preferred or the corrupted score drops beyond the $x^{+}$ noise band;
logical-error detection is defined identically on the logical condition.
Ablation $\Delta$ is the mean absolute verdict change under $x_\varnothing$
and $x'$. The VGS interval is a paired nonparametric bootstrap over instances
($2{,}000$ resamples, fixed seed): each resample draws instance indices with
replacement and recomputes
$\mathrm{mean}(\text{flip indicator})-\mathrm{mean}(\text{null indicator})$ on
those indices, preserving the per-instance pairing between the two terms.
Columns in \cref{tab:main} are rounded independently, so the reported gap can
differ
from the rounded difference of the reported GEDR and LEDR by $0.01$; this
occurs in seven cells.

The shortfall (the headline quantity built from GEDR and VGS together) has a
real per-cell bootstrap CI, computed and reported in \cref{sec:ordering}: only
$25$ of $36$ cells
exclude $0$, not the $35$ implied by point estimates alone.
GEDR and LEDR's own per-cell intervals, as opposed to the shortfall's, are
not yet shown in that table --- \texttt{groundjudge\_metrics.audit\_judge}
already returns \texttt{GEDR\_lo/hi} and \texttt{LEDR\_lo/hi} for every
cell, so adding them is a table-layout task (two more columns or
CI-annotated cells), not an outstanding analysis. A paired permutation test
of $\vgs\neq0$ per cell (each cell's flip- vs.\ null-indicator, $10{,}000$
resamples) plus \texttt{groundjudge\_metrics.holm\_reject} across all $36$
$p$-values gives $26/36$ cells significant at uncorrected $p<0.05$ and
$23/36$ surviving Holm--Bonferroni. The three cells that flip from marginal to
non-significant are Qwen2.5-VL/MathVista ($p=0.0154$ uncorrected),
InternVL3-8B/ScienceQA ($p=0.0145$) and Gemini-3.5-Flash/MathVista; no cell
that was non-significant uncorrected becomes significant. Reported in the main
paper immediately after ``Judges are not ignoring the image.''

Cells share instances within a dataset and every judge sees the same instances,
so no sign test on tabulated cells is fully valid; the three below are
descriptive and each is at best an upper bound on the evidence, reported here
rather than in the main paper because the pooled instance-level permutation
test (\cref{sec:ordering}) is the one test
the design actually supports and supersedes them. The primary pool gives
$18/18$, $p=7.6\times10^{-6}$; clustering by judge gives $9/9$, $p=0.004$;
treating all cells as exchangeable gives $32/33$ non-tied, $p=7.9\times10^{-9}$.
All three agree in direction with the valid pooled test (mean shortfall
$0.265$, permutation $p<5\times10^{-5}$, $20{,}000$ permutations) reported in
\cref{sec:ordering}.

\subsection{Reference preference accuracy, measured and not reported}
\label{sec:refacc}

We measured an in-house preference accuracy on the VL-RewardBench test split by
scoring each of the two candidate responses independently with each judge's
pointwise score and predicting the higher-scored one, because no judge in our
pool exposes a native pairwise call. We do not report those numbers. They do
not reproduce the published protocol: the benchmark paper reports $65.4\%$ for
GPT-4o where our protocol gives $0.799$, Llama-3.2-11B-Vision is reported in
the literature near $0.38$ where ours gives $0.766$, and the released split is
described as $1{,}250$ curated examples against the $1{,}247$ we scored.
Presenting a number labelled ``VL-RewardBench accuracy'' that a reviewer would
compare against the leaderboard would be misleading, and Proposition~A.2 means no
such number bounds grounding in any case. The argument in the main paper is
internal to the audit and does not use it. If a reference preference number is
wanted for a future version, MMRB2 is the current standard and reports frontier
judges $10$ to $25$ points below human agreement.

\subsection{Downstream harm: what exists and what does not}
\label{sec:downstream}

The selection experiment selects by argmax against the true image and
compares against an oracle
that always picks the clean trace. Because ``ungrounded-selected'' means ``did
not pick the clean trace,'' that rate and task accuracy sum to one by
construction; the main paper reports one column rather than two for this
reason. The reward-correlation analysis reuses each judge's main-audit
pointwise scores as rewards and introduces
no new judge calls, so it is a re-analysis of the main audit rather than an
independent
measurement.

Both experiments were in fact already collected for all eight judges on all
four datasets (\texttt{results/E2/table2\_e2.csv},
\texttt{results/E3/table2\_e3.csv}, $32$ rows each) --- the claim above that
only two judges' worth of the reward-correlation analysis exists was stale
relative to the released
artifacts, the same pattern found for the artifact control. \Cref{stab:harm}
reports the full eight-judge grid.

\subsection{Does GEDR predict downstream selection?}
\label{sec:gedrpredict}

Proposition~A.4 makes GEDR the probe defined on the unedited image, so we tested
whether GEDR predicts the ungrounded-selection rate where VGS does not. Over
the eight judges, Spearman correlation with the selection rate is $0.00$ on
ChartQA and $-0.42$ on ScienceQA for VGS, and $+0.60$ and $-0.66$ for GEDR.
GEDR is the stronger correlate in both datasets and pooled over the $16$ cells
($-0.59$ against $-0.44$), but its sign reverses between datasets, and exact
permutation tests over all $8!=40{,}320$ orderings give $p=0.12$ and $p=0.086$.
We therefore claim no predictive advantage. The reversal is driven by
Claude-Sonnet-4.5, which has the lowest ChartQA GEDR ($0.21$) and the best
selection rate ($0.11$) and is also the judge with the largest blank-ablation
delta: a judge that abstains rather than discriminates can rank a clean trace
first without detecting anything. Proposition~A.4 concerns pairwise detection
while a selection rate is an argmax over a pool; these are different
functionals and eight judges cannot separate them. We drop the predictive
claim and keep Proposition~A.4 as a scoping result.
Widening the judge pool to settle the question properly remains future
work.

\subsection{Experiments that would close the open items}
\label{sec:close}

\textbf{The magnitude sweep --- run, null result.} The ChartQA
construction edits a bar value and so admits a magnitude parameter.
Re-rendering the counterfactual at several magnitudes predicted a monotone
rise in VGS while GEDR, computed on the unedited image, stays flat. Run for
all 4 open-weight judges: the prediction does NOT hold. 3 of 4 judges show
a negative Spearman correlation between magnitude and the swept quantity
(monotone permutation $p=0.80$--$0.98$); only Llama-3.2-Vision reaches
significance ($p=0.019$) on a nearly flat, degenerate curve. See
\cref{sec:audit} for the full result reported as found, and
\cref{sfig:magnitude} for the curves themselves.

\textbf{The reasoning-mode experiment.} One extra
condition each. If a thinking mode raises $\pi_2$ without raising $\pi_1$ the
failure is arbitration; if neither moves it is perception. Not run ---
needs Qwen3-VL/InternVL3.5's Instruct/Thinking variants (B10, blocked on
new model access).

\textbf{The read-attribute experiment --- run.} Asked each judge to report
the attribute under $x$ and under $\tilde x$ and recorded whether it moves,
for all 4 open-weight judges on all 4 datasets. It exposes two violations of
\cref{ass:med}, both in scoring rather than perception and both in the secondary
pool: a MathVista answer-polarity bias and a ScienceQA verdict-binarization
effect. Neither is a grounding failure. Per-judge numbers in \cref{sec:readattr}.

\textbf{The artifact control $X^{\dagger}$ --- run for eight of nine
judges.} Records already existed for the eight original judges; VGS is
recomputed with them in \cref{sec:artifact}. The mean absolute difference from
the
$X^{+}$-based value is $0.025$ across those $32$ cells, so artifact sensitivity
is a small effect. Gemini-3.5-Flash was not covered.

\begin{figure}[t]
\centering
\kilgraphics[0.72\linewidth]{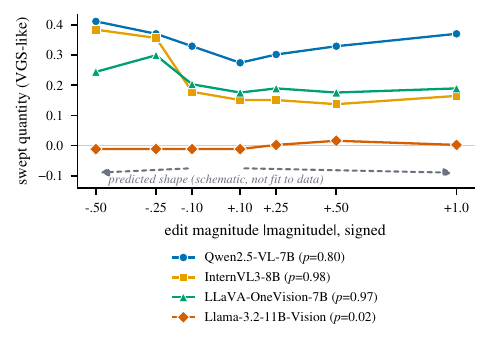}
\caption{\textbf{The magnitude sweep, and its null result.} Each
curve is the per-magnitude VGS-like quantity of \cref{sec:close}
(\cref{sec:audit}, ``What would close this, and what we ran''), for the four
open-weight judges on
ChartQA's 511 magnitude-tagged counterfactual instances. The prediction was
a monotone rise with $|{\rm magnitude}|$ on both sides of zero (shown
schematically, low and unfit to any curve); three of four judges instead
fall or stay flat, with monotone permutation $p=0.80$ to $0.98$.
Llama-3.2-11B-Vision alone reaches significance ($p=0.019$), on a curve
whose seven points span only $0.027$, consistent with the
verdict-binarization insensitivity this judge shows elsewhere in the audit
(\cref{sec:readattr}). Recomputed directly from the shipped per-instance
verdicts by \texttt{make\_figS2.py}, which asserts these exact figures at
build time and fails if they drift.}
\label{sfig:magnitude}
\end{figure}

\subsection{Blank-image ablation, and probe disagreement on ranking}
\label{sec:blank}

The probes disagree on the ordering of judges, not only cell by cell. Ranking
the nine by mean VGS and by mean GEDR gives Spearman $0.82$, but Qwen2.5-VL-7B
moves from third on VGS to seventh on GEDR, so which judge an auditor selects
depends on which probe is reported.

Averaged over four datasets the four proprietary judges change their verdict on
$0.77$ of instances when the image is blanked but only $0.57$ when it is
replaced by a mismatched one (the three original proprietary judges alone give
$0.77$ against $0.54$); the four open-weight generalist judges reverse this,
$0.24$ against $0.41$, and Claude-Sonnet-4.5 is extreme at $0.90$ against
$0.33$. VisualPRM, scored through a step-level interface, sits between the
groups at $0.73$ against $0.60$.
Reacting strongly to \emph{no} image and weakly to a \emph{wrong} one is a
response to the presence of visual input, not its content. Mean blank-ablation
$\Delta$ correlates with mean VGS at $r=0.54$ over the nine
judges, indistinguishable at $n=9$ from either $0$ or a strong correlation.
Contrasting judge behaviour with and without the image is the instrument used by
\citet{zou2026informativeness}; our ablation is why we do not treat it as a
grounding measure.

\subsection{The read-attribute experiment: where Assumption 1 is violated}
\label{sec:readattr}

The read-attribute experiment records the judge's reported attribute value
under $x$ and under $\tilde x$,
which makes $\pi_1$ and $\pi_2$ observed rather than inferred. Checking
$\gedr=\pi_1$ against that direct measurement, the largest discrepancies are on
ScienceQA: Llama-3.2-Vision at $-0.86$ and InternVL3 at $-0.50$. Both trace to
verdict binarization rather than to perception. Llama's raw score does move under
the edit, from a mean of $8.2$ to $6.5$ on ScienceQA, but it rarely crosses the
pointwise threshold at which the verdict is binarized, so $\vgs$ barely registers
a shortfall that $\gedr$ and $\pi_1$ both see clearly. Separately, on MathVista
three of four judges show an answer-polarity bias in scoring: a clean trace
concluding ``yes'' scores about $7$ to $9$ out of $10$ while an equally correct
clean trace concluding ``no'' scores only about $3$ to $5$. That makes
$\varphi_r$ depend on the trace's literal content, which \cref{ass:med} does not
allow. Both effects fall in the secondary pool.
We report them as named caveats rather than excluding the affected cells or
rescoring with a bias-controlled prompt, either of which is future work and not a
correction to what was measured.

\subsection{A held-out test of the ordering}
\label{sec:heldout}

The audit's own ordering result is retrodiction: \cref{cor:order} was
derived to explain a pattern already visible in the eight-judge (later
nine-judge) pool. The genuine fix is to predict the ordering for a judge never
touched by any prior step, before querying it, and check afterward. We did
this once. Qwen3-VL-8B-Instruct was chosen before any call was made, for a
reason recorded in writing at the time (same size class as the existing
open-weight judges, so the comparison is not confounded by scale, and it
directly answers an earlier finding in this project that Qwen3.x was not
vision-capable on OpenRouter --- it now is). The prediction, VGS $\le$ GEDR on
all four datasets under \cref{cor:order}'s side condition, was written
and timestamped before the first query, together with a commitment to report
every cell regardless of outcome.

\begin{table}[t]
\centering
\caption{\textbf{Held-out judge (Qwen3-VL-8B-Instruct).} The ordering holds
strictly on all four cells, and the edit-miss rate is non-negative throughout.}
\label{tab:heldout}
\small
\begin{tabular}{@{}l S[table-format=1.3] S[table-format=1.3] S[table-format=1.3] S[table-format=1.3] S[table-format=1.3]@{}}
\toprule
\kilth{Dataset} & {\kilth{VGS}} & {\kilth{GEDR}} & {\kilth{$\Delta_{\mathrm{null}}$}} & {\kilth{shortfall}} & {\kilth{edit-miss}} \\
\midrule
ChartQA & 0.040 & 0.210 & 0.110 & 0.170 & 0.060 \\
HallusionBench & 0.308 & 0.452 & 0.038 & 0.144 & 0.106 \\
MathVista & 0.000 & 0.375 & 0.125 & 0.375 & 0.250 \\
ScienceQA & 0.364 & 0.727 & 0.045 & 0.363 & 0.318 \\
\bottomrule
\end{tabular}
\end{table}

The ordering holds strictly on all four cells, including both primary-pool
datasets, and the edit-miss rate $\pi_1-\pi_2=(\gedr-\vgs)-\Delta_{\mathrm{null}}$
is non-negative on every cell, so the side condition the prediction depended on
is satisfied throughout, not merely the surface inequality. Per-instance
verdicts for this judge are in the Code and Data Supplement alongside the
other nine. One held-out judge is a single trial, not a validation set, and it
does not retire the retrodiction label for the audit's original ordering claim
--- but the mechanism the identity depends on has now been tested once in the
condition that actually falsifies it, and it did not fail.

%%%%%%%%%%%%%%%%%%%%%%%%%%%%%%%%%%%%%%%%%%%%%%%%%%%%%%%%%%%%%%%%%%%%%%%%%%%%%%%%
\section{Reproducibility checklist notes}
\label{sec:checklist}

Theoretical contributions: \cref{def:vgs}, \cref{lem:perc}, \cref{thm:bound},
\cref{cor:order} and \cref{cor:onesided} and \cref{prop:gap} are stated formally
in the main paper with their
assumptions in the statements, and proved in \cref{sec:proofs}; the five routine
supporting results are stated and proved in \cref{sec:support}. \Cref{ass:med}
is stated formally, and
\cref{sec:readattr} documents the two cells in which it is violated, both in the
secondary pool. \Cref{sec:sciqa} documents the ScienceQA leakage that bears on
\cref{ass:match}(a). Computational experiments: datasets are public and cited; the
construction pipeline, injected errors, image edits, judge prompts, metric
code and per-instance verdicts are in the Code and Data Supplement. The number
of instances per dataset is stated in the main paper, uncertainty is reported
for VGS and for the shortfall but not for GEDR or LEDR individually, and no
hyperparameter search was performed because the diagnostic has no trained
component.

Dataset licenses: ChartQA is GPL-3.0 (per its official repository's LICENSE
file); HallusionBench is BSD 3-Clause (per its official repository's
LICENSE.md); MathVista is CC BY-SA 4.0 for its own contributions (IQTest,
FunctionQA, Paper subsets, filtering and metadata), per its Hugging Face
dataset card, though the FigureQA-sourced subset we use is stated there to
retain the copyright and license terms of its original source rather than
MathVista's own; ScienceQA's data (as distinct from its code, which is MIT)
is CC BY-NC-SA 4.0, per its repository's LICENSE-DATA file --- non-commercial
use only; and GQA is CC BY 4.0 for its own annotations (questions, programs,
scene graphs), per its official download page, while the underlying images
and base scene graphs derive from COCO, Flickr and Visual Genome under those
sources' own separate terms, not superseded by GQA's license.

The three tables below are the full eight- and nine-judge grids
\cref{sec:stats}, \cref{sec:downstream} and \cref{sec:gedrpredict} summarize in
the text: downstream selection and
reward-groundedness correlation (\cref{stab:harm} and \cref{stab:e3}), and
the complete audit (\cref{stab:main}).

\begin{table}[t]
\centering
\caption{\textbf{Downstream selection and reward-groundedness correlation,
all eight judges, all four datasets.} The
non-clean selection rate is the rate at which the judge's argmax pick
against the true image is not the clean trace, from an independent
candidate pool built from each dataset's own ground-truth vocabulary; task
accuracy is one minus this rate by construction, so only one column is
given per dataset. The reward column is the Pearson correlation between
each judge's own main-audit pointwise score (as a reward) and ground-truth
groundedness, averaged over the four per-dataset values reported in
\cref{stab:e3}. Every judge selects a non-clean candidate at a nonzero rate on at
least one dataset; on both MathVista and ScienceQA several judges (four of
eight on ScienceQA) never do, and those nulls are reported as found. Neither
column orders judges the way mean VGS does: the reward correlation does not
significantly correlate with mean VGS (Spearman $\rho=0.62$, $p=0.10$,
$n=8$) or with mean selection rate ($\rho=-0.26$, $p=0.53$) across the
eight-judge pool.}
\label{stab:harm}
\small
\begin{tabular}{@{}l cccc c@{}}
\toprule
& \multicolumn{4}{c}{\kilth{Non-clean selection rate}} & \kilth{Reward} $\rho$\\
\cmidrule(lr){2-5}\cmidrule(l){6-6}
\kilth{Judge} & ChartQA & HallB. & MathV. & SciQA & mean\\
\midrule
LLaVA-OneVision & 0.20 & 0.27 & 0.75 & 0.05 & 0.40\\
Qwen2.5-VL-7B & 0.23 & 0.11 & 0.06 & 0.00 & 0.34\\
InternVL3-8B & 0.25 & 0.18 & 0.19 & 0.09 & 0.31\\
Llama-3.2-Vision & 0.55 & 0.37 & 0.00 & 0.00 & 0.15\\
GPT-4o & 0.20 & 0.07 & 0.38 & 0.00 & 0.32\\
Gemini-2.5-Flash & 0.19 & 0.06 & 0.31 & 0.00 & 0.33\\
Claude-Sonnet-4.5 & 0.08 & 0.13 & 0.12 & 0.23 & 0.41\\
VisualPRM-8B-v1.1 & 0.42 & 0.15 & 0.06 & 0.23 & 0.45\\
\bottomrule
\end{tabular}
\end{table}

\begin{table}[t]
\centering
\caption{\textbf{Reward-correlation analysis: reward-groundedness
correlation, per dataset.} Pearson
correlation between each judge's own main-audit pointwise score, used as a
reward,
and ground-truth groundedness (whether the trace is clean), reusing the main
audit's
scores with no new judge calls. Row means appear in \cref{stab:harm}'s Reward
column.}
\label{stab:e3}
\small
\begin{tabular}{@{}l cccc@{}}
\toprule
\kilth{Judge} & ChartQA & HallB. & MathV. & SciQA\\
\midrule
Qwen2.5-VL-7B & 0.33 & 0.31 & 0.17 & 0.54 \\
InternVL3-8B & 0.33 & 0.26 & 0.21 & 0.46 \\
LLaVA-OneVision & 0.40 & 0.22 & 0.00 & 0.47 \\
Llama-3.2-Vision & 0.02 & 0.00 & 0.00 & 0.56 \\
VisualPRM-8B-v1.1 & 0.34 & 0.45 & 0.53 & 0.46 \\
GPT-4o & 0.30 & 0.35 & 0.12 & 0.50 \\
Gemini-2.5-Flash & 0.28 & 0.36 & 0.19 & 0.48 \\
Claude-Sonnet-4.5 & 0.38 & 0.49 & 0.42 & 0.36 \\
\bottomrule
\end{tabular}
\end{table}

\begin{table*}[t]
\centering\footnotesize
\caption{\textbf{The full GroundJudge audit, all $36$ judge$\times$dataset
cells.} Instance counts are ChartQA $n{=}100$, HallusionBench $n{=}104$,
MathVista $n{=}16$, ScienceQA $n{=}22$. ChartQA and HallusionBench are the
primary pool of the main paper ($18$ cells); MathVista and ScienceQA are
secondary and carry nothing in the headline. VGS carries a $95\%$ paired
nonparametric bootstrap interval ($2{,}000$ resamples over instances).
$\Delta_{\mathrm{null}}$ is the flip rate under the truth-preserving
re-encoding $x^+$. The gap is the quantity \cref{cor:order} predicts
non-negative under its side condition; columns are rounded independently, so it
can differ from the rounded difference by $0.01$. The ablation columns are
binarized-verdict disagreement rates, not mean absolute change on the raw score.
\Cref{sec:stats} gives the full inferential detail: per-cell CIs for the
shortfall, the pooled instance-level test, and per-cell permutation tests with
Holm--Bonferroni correction.}
\label{stab:main}
\setlength{\tabcolsep}{4.5pt}\renewcommand{\arraystretch}{1.06}
\begin{tabular}{@{}ll cccc c cc@{}}
\toprule
& & \multicolumn{4}{c}{\kilth{Grounding and detection}} & \kilth{Gap} & \multicolumn{2}{c}{\kilth{Ablation} $\Delta$}\\
\cmidrule(lr){3-6}\cmidrule(lr){7-7}\cmidrule(l){8-9}
\kilth{Judge} & \kilth{Dataset} & VGS (95\% CI) & GEDR & LEDR & $\Delta_{\mathrm{null}}$ & GEDR$-$VGS & blank & mismatch\\
\midrule
\addlinespace[1pt]\multicolumn{9}{@{}l}{\itshape Open-weight judges}\\[1pt]
\quad \multirow{4}{*}{LLaVA-OneVision} & ChartQA & 0.21 \ci{0.13,0.30} & 0.65 & 0.57 & 0.03 & 0.44 & 0.19 & 0.49\\
& HallusionBench & 0.15 \ci{0.09,0.23} & 0.56 & 0.26 & 0.01 & 0.41 & 0.31 & 0.22\\
& MathVista & 0.31 \ci{0.12,0.56} & 0.50 & 0.38 & 0.00 & 0.19 & 0.19 & 0.44\\
& ScienceQA & 0.59 \ci{0.41,0.77} & 0.86 & 0.36 & 0.00 & 0.27 & 0.18 & 0.50\\
\quad \multirow{4}{*}{Qwen2.5-VL-7B} & ChartQA & 0.49 \ci{0.39,0.59} & 0.63 & 0.63 & 0.00 & 0.14 & 0.34 & 0.63\\
& HallusionBench & 0.26 \ci{0.16,0.36} & 0.54 & 0.66 & 0.04 & 0.28 & 0.42 & 0.32\\
& MathVista & 0.44 \ci{0.19,0.69} & 0.38 & 0.38 & 0.00 & $-$0.06$^{\ddagger}$ & 0.12 & 0.81\\
& ScienceQA & 0.77 \ci{0.59,0.95} & 0.95 & 0.23 & 0.00 & 0.18 & 0.27 & 0.59\\
\quad \multirow{4}{*}{InternVL3-8B} & ChartQA & 0.29 \ci{0.20,0.38} & 0.66 & 0.55 & 0.00 & 0.37 & 0.35 & 0.57\\
& HallusionBench & 0.15 \ci{0.07,0.24} & 0.58 & 0.47 & 0.04 & 0.43 & 0.47 & 0.33\\
& MathVista & 0.06 \ci{0.00,0.19} & 0.56 & 0.75 & 0.12 & 0.50 & 0.56 & 0.81\\
& ScienceQA & 0.32 \ci{0.14,0.50} & 0.77 & 0.64 & 0.00 & 0.45 & 0.23 & 0.45\\
\quad \multirow{4}{*}{Llama-3.2-Vision} & ChartQA & $-$0.05 \ci{$-$0.12,0.01} & 0.51 & 0.27 & 0.08 & 0.56 & 0.07 & 0.26\\
& HallusionBench & 0.02 \ci{$-$0.02,0.06} & 0.41 & 0.12 & 0.05 & 0.39 & 0.08 & 0.07\\
& MathVista & 0.00 \ci{0.00,0.00} & 0.00 & 0.00 & 0.00 & 0.00 & 0.00 & 0.00\\
& ScienceQA & 0.00 \ci{0.00,0.00} & 1.00 & 0.23 & 0.00 & 1.00 & 0.00 & 0.00\\
\addlinespace[1pt]\multicolumn{9}{@{}l}{\itshape Proprietary judges}\\[1pt]
\quad \multirow{4}{*}{GPT-4o} & ChartQA & 0.27 \ci{0.16,0.38} & 0.58 & 0.47 & 0.08 & 0.31 & 0.46 & 0.62\\
& HallusionBench & 0.38 \ci{0.28,0.47} & 0.63 & 0.13 & 0.04 & 0.25 & 0.35 & 0.42\\
& MathVista & 0.12 \ci{$-$0.19,0.44} & 0.44 & 0.50 & 0.25 & 0.32 & 0.62 & 0.62\\
& ScienceQA & 0.95 \ci{0.86,1.00} & 0.95 & 0.14 & 0.00 & 0.00 & 0.95 & 1.00\\
\quad \multirow{4}{*}{Gemini-2.5-Flash} & ChartQA & 0.38 \ci{0.26,0.49} & 0.48 & 0.71 & 0.12 & 0.10 & 0.74 & 0.72\\
& HallusionBench & 0.30 \ci{0.19,0.40} & 0.62 & 0.37 & 0.04 & 0.32 & 0.77 & 0.30\\
& MathVista & 0.25 \ci{0.06,0.50} & 0.50 & 0.62 & 0.06 & 0.25 & 0.69 & 0.62\\
& ScienceQA & 0.91 \ci{0.77,1.00} & 1.00 & 0.09 & 0.00 & 0.09 & 1.00 & 0.91\\
\quad \multirow{4}{*}{Claude-Sonnet-4.5} & ChartQA & 0.04 \ci{$-$0.01,0.10} & 0.21 & 0.82 & 0.04 & 0.17 & 0.95 & 0.35\\
& HallusionBench & 0.19 \ci{0.12,0.27} & 0.71 & 0.67 & 0.06 & 0.52 & 0.84 & 0.34\\
& MathVista & 0.00 \ci{$-$0.19,0.19} & 0.12 & 0.88 & 0.06 & 0.12 & 0.81 & 0.06\\
& ScienceQA & 0.55 \ci{0.36,0.73} & 0.95 & 0.14 & 0.00 & 0.40 & 1.00 & 0.55\\
\addlinespace[1pt]\multicolumn{9}{@{}l}{\itshape Trained process reward model}\\[1pt]
\quad \multirow{4}{*}{VisualPRM-8B-v1.1} & ChartQA & 0.54 \ci{0.44,0.64} & 0.69 & 0.58 & 0.00 & 0.15 & 0.61 & 0.59\\
& HallusionBench & 0.44 \ci{0.35,0.55} & 0.89 & 0.83 & 0.04 & 0.45 & 0.64 & 0.29\\
& MathVista & 0.69 \ci{0.44,0.88} & 0.94 & 0.94 & 0.00 & 0.25 & 0.94 & 0.75\\
& ScienceQA & 0.59 \ci{0.32,0.86} & 0.86 & 0.77 & 0.09 & 0.27 & 0.73 & 0.77\\
\addlinespace[1pt]\multicolumn{9}{@{}l}{\itshape Current-generation frontier judge (B10, added 2026-08-05)}\\[1pt]
\quad \multirow{4}{*}{Gemini-3.5-Flash} & ChartQA & 0.61 \ci{0.51,0.70} & 0.72 & 0.55 & 0.00 & 0.11 & 0.62 & 0.62\\
& HallusionBench & 0.55 \ci{0.45,0.63} & 0.94 & 0.07 & 0.00 & 0.39 & 0.76 & 0.30\\
& MathVista & 0.62 \ci{0.31,0.88} & 0.81 & 0.44 & 0.12 & 0.19 & 0.81 & 0.62\\
& ScienceQA & 1.00 \ci{1.00,1.00}$^{\P}$ & 1.00 & 0.14 & 0.00 & 0.00 & 1.00 & 1.00\\
\bottomrule
\end{tabular}
\tabnote{$^{\ddagger}$The single cell that violates the ordering.
$^{\P}$Gemini-3.5-Flash's $\vgs=\gedr=1.00$ is a genuine perfect-detection
ceiling, confirmed from per-instance verdicts showing varying scores rather than
assumed from the point estimate. Contrast Llama-3.2-Vision on MathVista, which
returns an identical score for every condition on all $16$ instances: that row is
$0$ by degeneracy, not by image-blindness.}
\end{table*}

\FloatBarrier

\end{document}